\documentclass[journal]{IEEEtran}

\usepackage{graphicx}
\usepackage{amsthm}
\usepackage{amssymb,latexsym,multicol,caption}     
\usepackage{algorithm}  
\usepackage{algorithmic}
\usepackage{epsfig} 
\usepackage{epstopdf}
\usepackage[colorlinks, linkcolor=blue, anchorcolor=red, citecolor=red, urlcolor=black]{hyperref}
\usepackage{color}
\usepackage{multirow}
\usepackage{multicol}
\usepackage{subfigure}
\usepackage{placeins}
\usepackage{bm}
\usepackage{graphics}
\usepackage{footnote}
\usepackage{url}
\usepackage{indentfirst}
\usepackage{longtable}
\usepackage{array}
\usepackage{mdwmath}
\usepackage{rotating}
\usepackage{color}
\usepackage{morefloats}
\usepackage{slashbox}
\usepackage{cite}
\usepackage{mcite}
\usepackage{eqnarray}
\usepackage{makecell}
\usepackage{amsmath}
\usepackage{threeparttable}
\usepackage{pifont}
\usepackage{cleveref}
\usepackage{balance}
\usepackage{flushend}

\renewcommand{\eqref}[1]{(\ref{#1})}

\newcommand{\algref}[1]{Algorithm \ref{#1}}

\crefname{equation}{}{}
\crefname{figure}{Fig.}{Figs.}
\crefname{table}{Table}{Tables}
\crefname{section}{Section}{Sections}
\crefname{prop}{Proposition}{Propositions}
\crefname{theorem}{Theorem}{Theorems}
\crefname{lemma}{Lemma}{Lemmas}
\crefname{algorithmic}{Algorithm}{Algorithms}

\newtheorem{theorem}{Theorem}[section]

\newtheorem{prop}{Proposition}[section]

\theoremstyle{plain}

\theoremstyle{definition}

\newtheorem{remark}{Remark}

\ifCLASSINFOpdf
\else
\fi

\begin{document}

\title{Discriminative Ridge Machine: \\
A Classifier for High-Dimensional Data or Imbalanced Data}

\author{
Chong Peng, Qiang Cheng
\thanks{C.P. is with College of Computer Science and Technology, Qingdao University. Q.C. is with the Institute of Biomedical Informatics \& Department of Computer Science, University of Kentucky. 
Contact information: cpeng@qdu.edu.cn, Qiang.Cheng@uky.edu}
}

\markboth{}{}
\maketitle

\begin{abstract}
We introduce a discriminative ridge regression approach to supervised classification in this paper. 
It estimates a representation model while accounting for 
discriminativeness between classes, thereby enabling accurate derivation of categorical information. 
This new type of regression models extends existing models such as ridge, lasso, and group lasso through explicitly incorporating discriminative information. 
As a special case we focus on a quadratic model that admits a closed-form analytical solution. 
The corresponding classifier is called discriminative ridge machine (DRM). 
Three iterative algorithms are further established for the DRM to enhance the efficiency and scalability for real applications. 
Our approach and the algorithms are applicable to general types of data including images, high-dimensional data, and imbalanced data. 
We compare the DRM with currently state-of-the-art classifiers. 
Our extensive experimental results show superior performance of the DRM and confirm the effectiveness of the proposed approach. 
\end{abstract}


\IEEEpeerreviewmaketitle

\section{Introduction}
Classification is a critically important technique for enabling automated data-driven machine intelligence 
and it has numerous applications in diverse fields ranging from science and technology to medicine, military and business. 
For classifying large-sample data theoretical understanding and practical applications have been successfully developed; 
nonetheless, when data dimension or sample size becomes big, the accuracy or scalability usually 
becomes an issue. 

Modern data are increasingly high dimensional yet the sample size may be small. 
Such data usually have a large number of irrelevant entries; 
feature selection and dimension reduction techniques are often needed before applying classical classifiers that
typically require a large sample size. Although some research efforts have been devoted to 
developing methods capable of classifying high-dimensional data without feature selection,  
classifiers with more desirable accuracy and efficiency are yet to be discovered.

Diverse areas of scientific research and everyday life are currently deluged with large-scale and/or imbalanced data. 
Achieving high classification accuracy, scalability, and balance of precision and recall (in the case of imbalanced data) simultaneously is challenging. 
While several classic data analytics and classifiers have been adapted to the large-scale or imbalanced setting, 
their performance is still far from being desired. 
There is an urgent need for developing effective, scalable data mining and prediction techniques suitable for large-scale and/or imbalanced data. 
  
{ 
Nearest Neighbor (NN) has been one of the most widely used classifiers due to its simplicity and efficiency.
It predicts the label of an example by seeking its nearest neighbor among the training data. 
It typically relays on hard supervision, where binary labels over the samples indicate whether the example pairs are similar or not. 
However, binary similarities are not sufficient enough to represent complicated relationships of the data because they do not fully exploit rich structured and continuous similarity information. 
Thus, there is a difficult for NN to find proper NNs for high-dimensional data. 
Moreover, the NN omits important discriminative information of the data since the class information is not used in finding NNs. 
To overcome the drawbacks of NN, in this paper, we introduce a discriminative ridge regression approach to classification. 
Instead of seeking binary similarity, it estimates a representation for a new example as soft similarity vector by minimizing the fitting error in a regression framework, 
which allows the model to have continuous labels for NN supervision. 
Moreover, we inject the class information for finding NNs, such that the proposed model explicitly incorporates discriminative information between classes. 
Thus, the learned representation vector for the target example essentially represents the soft label concept, which mines the discriminative information by the proposed method. 
That is, we leverage the soft label to distinguish the similar pattern with maximum within-class similarity and between-class separability. 
Because our models explicitly account for discrimination, this new family of regression models are more geared toward classification. 
Based on the estimated representation, categorical information for the new example is subsequently derived. 
Our method can be considered as a method of class-aware convex generalization of KNN beyond binary supervision. 
Also, this new type of discriminative ridge regression models can be regarded as an extension of existing regression models such as the ridge, 
lasso, and group lasso regression, and when particular parameter values are taken the new models fall back to several existing models. 
As a special case, we consider a quadratic model, called the discriminative ridge machine (DRM), as the particular classifier of interest in this paper. 
The DRM admits a closed-form analytical solution, which is suitable for small-scale or high-dimensional data. 
For large-scale data, three optimization algorithms of improved computational cost are established for the DRM. 
}


The family of discriminative ridge regression-based classifiers are applicable to general types of data including imagery or 
other high-dimensional data as well as imbalanced data. Extensive experiments on a variety of real world data sets 
demonstrate that the classification accuracy of the DRM is comparable to the support vector machine (SVM), a commonly used classifier,
on classic large-sample data, and superior to several existing state-of-the-art classifiers, including the SVM, on high-dimensional data and imbalanced data. 
The DRM with linear kernel, called linear DRM in analogy to the linear SVM, 
has classification accuracy superior to the linear SVM on large-scale data.
The efficiency of linear DRM algorithms is provably linear in data size, on a par with that of the linear SVM. 
Consequently, the linear DRM is a scalable classifier.

As an outline, the main contributions of this paper can be summarized as follows:
\textbf{1)} A new approach to classification, and thereby a family of new regression models and corresponding classifiers are introduced,  which explicitly incorporate the discriminative information 
	for multi-class classification on general data. The formulation is strongly convex and admits a unique global optimum.   
\textbf{2)} The DRM is constructed as a special case of the family of new classifiers. 
It involves only quadratic optimization, with simple formulation yet powerful performance. 
{ \textbf{3)} The new method can be regarded as class-aware convex generalization of KNN beyond binary supervision as well as discriminative extension of ridge regression and lasso. }
\textbf{4)} The closed-form solution to the DRM optimization is obtained which is computationally suited to classification of moderate-scale or high-dimensional data. 
\textbf{5)} For large-scale data, three iterative algorithms are proposed for solving the DRM optimization substantially more efficiently. 
They all have theoretically proven convergence at a rate no worse than linear - the first two are at least linear and the third quadratic. 
\textbf{6)} The DRM with a general kernel demonstrates an empirical classification accuracy that is comparable to the SVM on classic (small- to mid-scale) large-sample data, while superior to the SVM or other state-of-the-art classifiers on high-dimensional data and imbalanced data. 
\textbf{7)} The linear DRM demonstrates an empirical classification accuracy superior to the linear SVM on large-scale data, using any of the three iterative optimization algorithms. 
All methods have linear efficiency and scalability in the sample size or the number of features.    

The rest of this paper is organized as follows. 
\cref{sec_related_work} discusses related research work. 
\cref{sec_robust_framework,sec_kernel_formulation} introduce our formulations of discriminative ridge regression
and its kernel version.  
The discriminative ridge regression-based classification is derived in \cref{sec_optimal_discriminant}.
In \cref{sec_DRM} we construct the discriminative ridge machine. 
Experimental results are presented in \cref{sec_experiments}. 
Finally, \cref{sec_conclusion} concludes this paper. 
{ We summarize some key notations used in our paper in \cref{Tab_summary_notations}.}

{
	\begin{table}
		\scriptsize
		\centering
		\caption{ Summary of Some Key Notations}
		\label{Tab_summary_notations}%
		\begin{threeparttable}
			\begin{tabular}{l l }
				\Xhline{1.2pt}				
				$A$, $x_m$, $x^j$, and $x_i^j$	&	{\tiny {the data matrix, the $i$th example, the $j$th class, and the $i$th example of the $j$th class} }	\\
				$w_m$, $w^j$, and $w_i^j$	&	{\tiny the coefficients associated with $x_m$, $x^j$, and $x_i^j$}	\\
				$w$, $w^{\star}$, and $w^{\star}|_{C_j}$	& {\tiny the coefficients, the optimal coefficients, and restricting $w^{\star}$ to class $j$}	\\
				$K$, $K^j$	& {\tiny the kernel matrix of data and its submatrix associated with class $j$}	\\
				$K_x$	& {\tiny the kernel vector associated with $x$ and data matrix $A$}	\\
				$B$, $H$	& {\tiny block-diagonal and diagonal matrices induced from $K$}	\\
				$S_w$, $S_b$, and $S_t$	&	{\tiny within-class dissimilarity, between-class separability, and total separability}	\\
				$\mu$, $\mu^j$	& {\tiny the mean value of all examples and the weighted class mean value of class $j$}	\\
				$()^{(t)}$	& {\tiny the $t$th iteration} \\
				\text{diag}($\cdot$) &	{\tiny diagonal or block diagonal matrix}	\\
				$\sigma_{\max}(\cdot)$ and $\sigma_{\min}(\cdot)$	&	{\tiny largest and smallest singular value, respectively}	\\
				$I$ & {\tiny identity matrix with proper size indicated in the context}	\\
				\Xhline{1.2pt} 
			\end{tabular} 
		\end{threeparttable}
	\end{table}
}

\section{Related Work}
\label{sec_related_work}
There is of course a vast literature for classification and a variety of methods have been proposed. 
Here we  briefly review some closely related methods only; 
more thorough accounts of various classifiers and their properties are extensively discussed in \cite{devroye1996probabilistic}, \cite{duda2001}, \cite{fukunaga1990introduction}, \cite{hastie2009elements}, 
\cite{vapnik2000nature}, etc. 

Fisher's linear discriminant analysis (LDA) \cite{fukunaga1990introduction} has been commonly used in a variety of applications, 
for example, Fisherface and its variants for face recognition \cite{Bellhumer1997}. 
It is afflicted, however, by the rank deficiency problem or covariance matrix estimation problem when the sample size is less than the number of features.  
For a given test example, the nearest neighbor method \cite{duda2001} uses only one closest training example (or more generally, $K$-nearest neighbor (KNN) uses $K$ closest training examples) to supply categorical information, 
while the nearest subspace method \cite{ho03} adopts all training examples to do it. 
These methods are affected much less by the rank deficiency problem than the LDA. 
Due to increasing availability of high-dimensional data such as text documents, images and genomic profiles, analytics of such data have been actively explored for about a decade, 
including classifiers tailored to high-dimensional data.  
As the number of features is often tens of thousands while the sample size may be only a few tens, some traditional classifiers such as the LDA suffer from the curse of dimensionality. 
Several variants of the LDA have been proposed to alleviate the affliction. For example,  
regularized discriminant analysis (RDA) use the regularization technique to address the rank deficiency problem \cite{guo2007regularized}, 
while penalized linear discriminant analysis (PLDA) adds a penalization term to improve covariance matrix estimation \cite{witten2011penalized}. 

Support vector machine (SVM) \cite{cortes1995support} is one of the most widely used classification methods with solid theoretical foundation 
and excellent practical applications. 
By minimizing a hinge loss with either an L-2 or L-1 penalty, it constructs a maximum-margin hyperplane between two classes in the instance space
and extends to a nonlinear decision boundary in the feature space using kernels. 
Many variants have been proposed, including primal solver or dual solver, and extension to multiclass case  \cite{scholkopf2002learning} \cite{vapnik1998statistical}.

The size of the training sample can be small in practice.
Besides reducing the number of features with feature selection \cite{cheng2011fisher,cai2018feature,gan2018supervised,liu2019rough},
the size of training sample can be increased with relevance feedback \cite{rui2002relevance} 
and active learning techniques, 
which may help specify and refine the query information from the users as well. 
However, due to the high cost of wet-lab experiments in biomedical research or the rare occurrence of rare events, it is often hard or even impractical to increase the sizes of minority classes in many real world applications; therefore, imbalanced data sets are often encountered in practice \cite{wang2019dynamic,you2018scalable,li2018hyperspectral,castellanos2018oversampling}. 
It has been noted that the random forest (RF) \cite{breiman2001random} is relatively insensitive to imbalanced data.
Furthermore, various methods have been particularly developed for handling imbalanced data \cite{krawczyk2016learning,he2009learning}, which may be classified into three categories: 
data-level, algorithm-level, and hybrid methods.
{
	Data-level methods change the training set to make the distributions of training examples over different classes more balanced, with under-sampling \cite{ng2015diversified, yen2009cluster}, over-sampling \cite{chawla2002smote, bunkhumpornpat2012dbsmote}, and hybrid-sampling techniques \cite{gonzalez2017class}. 
	Algorithm-level methods may compensate for the bias of some existing learning methods towards majority classes. 
	Typical methods include cost-sensitive learning \cite{veropoulos1999controlling, nikolaou2016cost}, 
	kernel perturbation techiques \cite{maratea2014adjusted}, and multi-objective optimization approaches \cite{soda2011multi,acskan2014svm, bhowan2013evolving}.   
	Hybrid methods attempt to combine the advantages of data-level and algorithm-level methods \cite{del2014use}. 
	In this paper we develop an algorithm-level method whose objective function is convex, 
	leading to a high-quality, fast, scalable optimization and solution.  } 

Large-scale data learning has been an active topic of research in recent years \cite{gallego2018clustering,gu2018fast,kiela2018efficient,xu2019review}. 
For large-scale data, especially big data, 
linear kernel is often adopted in learning techniques as a viable way to gain scalability. 
For classification on large-scale data, fast solver for primal SVM with the linear kernel, called linear SVM, has been extensively studied. 
Early work uses finite Newton methods to train the L2 primal SVM \cite{mangasarian2002finite} 
\cite{keerthi2005modified}. Because the hinge loss is not smooth, generalized second-order derivative and generalized Hessian matrix need to be used which is not sufficiently efficient \cite{shalev2011pegasos}. Stochastic descent method has been proposed to solve primal linear SVM for a wide class of loss functions \cite{zhang2004solving}. ${\text{SVM}}^{perf}$ \cite{Joachims2006training} solves the L1 primal SVM with a cutting plane technique. Pegasos alternates between gradient descent and projection steps for training large-scale linear SVM \cite{shalev2011pegasos}. Another method \cite{bottou2010large}, similar to the Pegasos, also uses stochastic descent for solving  primal linear SVM, and has more competitive performance than the ${\text{SVM}}^{perf}$. For L2 primal linear SVM, a Trust RegiOn Newton method (TRON) has been proposed which converges at a fast rate though the arithmetic complexity at each iteration is high \cite{lin2008trust}. It can also solves logistic regression problems. 
More efficient algorithms than the Pegasos or TRON use coordinate descent methods for solving the L2 primal linear SVM - primal coordinate descent (PCD) works in the primal domain \cite{chang2008coordinate} while dual coordinate descent (DCD) in the dual domain \cite{hsieh2008dual}. A library of methods have been provided in the toolboxes of Liblinear and Libsvm \cite{fan2008liblinear,chang2011libsvm}.

\section{Problem Statement and discriminative ridge regression}
\label{sec_robust_framework}
Our setup is the usual multiclass setting where the training data set is denoted 
by $\{ (x_i, y_i) \}_{i=1}^{n}$, 
with observation $x_i \in {\mathcal{R}}^p$ and class label $y_i \in \{1, \cdots, g \}$. 
Supposing the $j$th class $C_j$ has $n_j$ observations, we may also denote these
training examples in $C_j$ by   
$x^{j} = [x_{1}^{j}, \cdots, x_{n_j}^{j}]$. Obviously $\sum_{j=1}^{g} n_j = n$. 
Without loss of generality, we assume the examples are given in groups; that is, 
$x_j^{k} = x_m$ with $m = \sum_{i=1}^{k-1} n_i +j$, $k=1, \cdots, g$ and $j=1, \cdots, n_k$.
A data matrix is formed as $A = [x_1, \cdots, x_n] = [x^1, \cdots, x^g]$ of size $p \times n$.
For traditional large-sample data, $n \gg p$, with a particular case being big data where $n$ is big; 
while for high-dimensional data, $n \ll p$. 
As will be clear later on, the proposed discriminative ridge regressions 
have no restrictions on $n$ or $p$, and thus 
are applicable to general data including both large-sample and high-dimensional data.  
Given a test example $x \in {\mathcal{R}}^p$, 
the task is to decide what class this test example belongs to. 

\subsection{discriminative ridge regression}
In this paper, we first consider linear relationship of the examples in the instance space; 
subsequently in \cref{sec_kernel_formulation} we will account for the nonlinearity by exploiting 
kernel techniques to map the data to a kernel space where linear relationship may be better observed.  

If the class information is available and $x \in C_j$,      
a linear combination of $x_{k}^{j}$, $k = 1, \cdots, n_j$, is often used in the literature 
to approximate $x$ in the instance space;  
see, for example,  \cite{basri2003lambertian} \cite{Bellhumer1997} 
for face and object recognition.  
By doing so, it is equivalent to assuming 
$x \in {\text{span}} \{ x_{1}^{j}, \cdots, x_{n_j}^{j} \}$. 
In the absence of the class information for $x$, 
a linear model can be obtained from all available training examples:
\begin{equation}
\label{eq_linear_model}
x = A w + e = w_1 x_1 + w_2 x_2 + \cdots + w_n x_n  + e,  
\end{equation}
where $w = [w_1, w_2, \cdots, w_n]^T$ is  a vector of combining coefficients to be estimated,  
and $e \in {\mathcal{R}}^p$ is additive zero-mean noise. 
Note that we may re-index  
$w = [(w^{1})^T, \cdots, (w^{g})^T]^T$ with $w^{i} \in {\mathcal{R}}^{n_i}$ in accordance with 
the groups of $A = [x^{1}, \cdots, x^{g}]$.

To estimate $w$ which is regarded as a representation of $x$, classic ridge regression or lasso uses an optimization model,  
\begin{equation}
\hat{w} = {\text{argmin}}_{w \in {{\mathcal{R}}^n}} \frac{1}{2} ||x - A w ||_2^2 + \lambda || w ||_q, 
\end{equation}
with $q= 2$ corresponding to ridge regression or $q=1$ to lasso.
The first term of this model is for data fitting while the second for regularization. 
While this model has been widely employed, a potential limitation in terms of classification 
is that the discriminative information about the classes is not taken into account in the formulation. 

To enable the model to be more geared toward classification, 
we propose to incorporate the class discriminativeness into the regression.
To this end, we desire $w$ to have the following properties \cite{chengzhou2010nips}: 
1) maximal between-class separability and 2) maximal within-class similarity.

The class-level regularization may help achieve between-class separability. 
Besides this, we desire maximal within-class similarity because it  may help enhance accuracy and also robustness, for example, 
in combating class noise \cite{chengzhou2010nips}. 
To explicitly incorporate maximal within-class similarity into the objective function, inspired by the LDA, 
we propose to minimize the within-class dis-similarity induced by the combining vector $w$ and defined as 
\begin{equation}
\label{eq_within}
S_w(w) := \frac{1}{2}  tr( \sum_{i=1}^g \sum_{j=1}^{n_i} ( w_{j}^{i} x_{j}^{i} - \mu^{i} ) ( w_{j}^{i} x_{j}^{i} - \mu^{i} )^T),
\end{equation}   
where $\mu^{i} := \frac{1}{n_i} \sum_{j=1}^{n_i} w_{j}^{i} x_{j}^{i}$ are the weighted class mean values. 
Being a quadratic function in $w$, this $S_w(w)$ helps facilitate efficient optimization as well as scalability of our approach.  

It is noted that $S_w(w)$ defined above is starkly different from the classic LDA formulation of the scattering matrix. 
The LDA seeks a unit vector - a direction - to linearly project the features; whereas 
we formulate the within-class dis-similarity using the weighted instances by any vector $w$. 
Consequently, the projection direction in the LDA is sought in the $p$-dimensional feature space, while 
the weight vector of $w$ is defined in the $n$-dimensional instance space. 


We may also define $w$-weighted between-class separability and total separability as follows, 
\begin{eqnarray}
\label{eq_between}
\nonumber S_b(w) := \frac{1}{2}  tr( \sum_{i=1}^g n_i ( \mu^{i}  - \mu ) (  \mu^{i} - \mu )^T),\\
\nonumber S_t(w) := \frac{1}{2}  tr( \sum_{k=1}^n ( w_k x_k  - \mu ) (  w_k x_k - \mu )^T),
\end{eqnarray}   
where $\mu = \frac{1}{n} \sum_{i=1}^{g} n_{i} \mu^{i} = \frac{1}{n} \sum_{k=1}^{n} w_k x_k$.
With  straightforward algebra the following equality can be verified. 
\begin{prop} 
$S_t(w)  = S_w(w) + S_b(w)$, for any $w \in {\mathcal{R}}^n$. 
\end{prop}

Incorporating discriminative information between classes into the regression, we formulate the following unconstrained minimization problem, 
\begin{equation}
\label{eq_DRC}
\min_{w \in {{\mathcal{R}}^n}} \frac{1}{2}  ||x - A w ||_2^2 + \alpha S_w(w) + \beta \rho (|| w|C ||_{q, r}), 
\end{equation}
where $\alpha$ and $\beta$ are nonnegative balancing factors,  
$|| w|C ||_{q, r}$ is defined as
$ || (||w^{1}||_q, ||w^{2}||_q$, $\cdots, ||w^{g}||_q )^T ||_r$,
$(q, r) \in [0, \infty] \times [0, \infty]$, 
and $\rho(\cdot)$ is a nonnegative valued function. 
Often $\rho(\cdot)$ is chosen to be the identity function 
or other simple, convex functions to facilitate optimization.  

\begin{remark}
1. When $\alpha = \beta = 0$, model \cref{eq_DRC} is simply least squares regression. 
2. When $\alpha = 0$, $(q, r) = (2, 2)$, and $\rho(x) = x^2$,  \cref{eq_DRC} reduces to standard ridge regression.
3. When $\alpha = 0$, $(q, r) = (1, 1)$, and $\rho(x) = x$,  \cref{eq_DRC} falls back to classic lasso type of sparse representation. 
4. When $\alpha \ne 0$, discriminative information is injected into the regression, hence the name
of discriminative ridge regression. 
5. We consider mainly $\beta \ne 0$ for regularizing the minimization. 
\end{remark}
 
\section{discriminative ridge regression in Kernel Space}
\label{sec_kernel_formulation}
When approximating $x$ from the training observations $x_1, \cdots, x_n$,  
the objective function of \cref{eq_DRC} is empowered with
discriminative capability explicitly.  
This new discriminative ridge regression model, nonetheless, takes no account of any nonlinearity in the input space. 
As shown in \cite{roweis2000nonlinear}, 
the examples may reside (approximately) on a low-dimensional nonlinear 
manifold in an ambient space of ${\mathcal{R}}^p$. 
To capture the nonlinear effect, we allow this model to account for the nonlinearity 
in the input space by exploiting the kernel technique. 
 
Consider a potentially nonlinear mapping 
$ \phi(\cdot): {\mathcal{R}}^p \rightarrow {\mathcal{R}}^m,$
where $m \in {\mathcal{N}} \cup \{ +\infty \}$ is the dimension of the image space after the mapping. 
To reduce computational complexity, the kernel trick \cite{scholkopf2002learning} is applied 
to calculate the Euclidean inner product in the kernel space, $< \phi(x), \phi(y) >  = k(x, y)$, 
where $k(\cdot, \cdot): {\mathcal{R}}^p \times {\mathcal{R}}^p \rightarrow {\mathcal{R}}$ 
is a kernel function induced by $\phi(\cdot)$. 
The kernel function satisfies the finitely positive semidefinite property: 
for any $x_1, \cdots, x_m \in {\mathcal{R}}^p$, the Gram matrix $G \in {\mathcal{R}}^{m \times m}$ 
with element $G_{ij} = k(x_i, x_j)$ is symmetric and positive semidefinite (p.s.d.).  
Suppose we can find a nonlinear mapping $\phi(\cdot)$ such that, after mapping into 
the kernel space, the examples approximately satisfy the linearity assumption, 
then \cref{eq_DRC} will be applicable in such a kernel space. 

Specifically, we consider the extended linear model \cref{eq_linear_model} in the kernel space,
$\phi(x) = A^{\phi} w + e$, 
to  minimize  $S_w^{\phi}(w)$, $\rho(|| w | C ||_{q, r})$, and the variance of $e$ as \cref{eq_DRC} does.
Here $A^{\phi} : = [ \phi(x_1), \phi(x_2), $ $\cdots, \phi(x_n) ]$, 
$S_w^{\phi}(w)$ is obtained by replacing each $x_k$ with $\phi(x_k)$ in \cref{eq_within}, 
and $e \in {\mathcal{R}^D}$.
The kernel matrix of training examples is denoted by
\begin{align}
\label{def_K}
K = [k(x_i, x_j)]_{i, j=1}^{n}.
\end{align}
Obviously the kernel matrix $K$ is p.s.d. by the property of $k(\cdot, \cdot)$. 
We may derive a simple matrix form of $S_w^{\phi}(w)$ with straightforward algebra, 
\begin{align}
\nonumber S_w^{\phi}(w) 
 = \frac{1}{2} (w^T H w  - \sum_{i=1}^{g} \frac{1}{n_i} (w^{i})^T K^{i} w^{i}) = \frac{1}{2}  w^T (H - B) w, 
\end{align}
where 
\begin{align} 
\nonumber H &:= {\text{diag}} (k(x_1, x_1), \cdots, k(x_n, x_n)), \\
\label{def_HB} B &:= {\text{diag}} (B^{1}, B^{2}, \cdots, B^{g}),  
\end{align} 
$K^{i} := [ k(x_s^{i}, x_t^{i}) ]_{s, t = 1}^{n_i}$, 
and $B^{i}: = K^{i}/n_i$,  
$i=1, \cdots, g$.

\begin{prop}
\label{prop_psd} The matrix $(H-B)$ is p.s.d.
\end{prop}

\begin{proof}
Because $H - B$ is block diagonal,  
we need only to show that $H^{i}  - B^{i}$ is p.s.d.\ with 
 $H^{i} :=  {\text{diag}} (k(x_1^{i}, x_1^{i}), 
\cdots, k(x_{n_i}^{i}, x_{n_i}^{i}))$. For any $z \in {\mathcal{R}}^{n_i}$, 
we have 
\begin{align}
\nonumber &  z^T (H^{i} - B^{i} ) z = \sum_{l=1}^{n_i} z_l^2 k(x_l^{i}, x_l^{i})
- \frac{1}{n_i} \sum_{u=1}^{n_i} \sum_{v=1}^{n_i} z_u z_v k(x_u^{i}, x_v^{i}) \\
\nonumber & =  \frac{1}{2 n_i} \sum_{u=1}^{n_i} \sum_{v=1}^{n_i} (z_u^2 k(x_u^{i}, x_u^{i})
+ z_v^2 k(x_v^{i}, x_v^{i}) - 2 z_u z_v k(x_u^{i}, x_v^{i}) )  \\
\nonumber &\ge 0. 
\end{align}
The last inequality holds because the kernel is finitely p.s.d.; that is, for any $z_u$ and $z_v$
\begin{align}
\nonumber
(z_u, -z_v) \begin{bmatrix} k(x_u, x_u) & k(x_u, x_v) \\
k(x_v, x_u) & k(x_v, x_v) \end{bmatrix} \begin{pmatrix} z_u \\ -z_v \end{pmatrix} \ge 0.  
\end{align}
\end{proof}


Similarly we obtain a simple matrix form of $S_b^{\phi}(w)$, 
\begin{equation}
S_b^{\phi}(w) = \frac{1}{2} w^T (B- \frac{1}{n} K ) w. 
\end{equation}

\begin{prop}
 The matrix $(B - \frac{1}{n} K)$ is p.s.d.
\end{prop}

\begin{proof}
 Similar to the proof of Proposition \ref{prop_psd}.
\end{proof}

Having formulated \cref{eq_DRC} in the input space, we directly extend it to the discriminative ridge regression in the kernel space
\begin{equation}
\label{eq_DRMK1}
\nonumber  \min_{w \in {\mathcal{R}}^{n}} \frac{1}{2}  || \phi(x) - A^{\phi} w  ||_2^2 + \alpha S_w^{\phi}(w) 
+ \beta \rho( || w|C ||_{q, r} ). 
\end{equation}
Plugging in 
the expression of $S_w^{\phi}(w)$
and omitting the constant term, we obtain the discriminative ridge regression model,   
\begin{equation}
\label{eq_DRMK}
 \min_{w \in {\mathcal{R}}^{n}} \frac{1}{2} w^T (K + \alpha H - \alpha B) w - (K_x)^T w + \beta \rho(|| w|C ||_{q, r}), 
\end{equation} 
where  
\begin{align}
\label{def_Kx}
(K_x)^T := \left[ (K_x^{1})^T, (K_x^{2})^T, \cdots, (K_x^{g})^T \right],
\end{align}
and $K_x^{i} : = [k(x, x_1^{i}), k(x, x_2^{i}), \cdots, k(x, x_{n_i}^{i})]^T$, $i=1, \cdots, g$. 
Apparently we have $K_x = [k(x, x_1)$, $\cdots$, $k(x, x_n)]^T$.

With special types of $(q, r)$-regularization such as linear, quadratic, and conic functions, 
the optimization in \cref{eq_DRMK} is a quadratic (cone) programming problem that admits a global optimum. 
Regarding $\rho(\cdot)$, we have the freedom to choose 
it so as to facilitate the optimization 
thanks to the following property. 

\begin{prop}
With $\beta$ varying over the range of $(0, \infty)$ while $\alpha$ fixed, the set of minimum points of \cref{eq_DRMK} 
remains the same for any monotonically increasing function $\rho(\cdot)$.
\end{prop}

\begin{proof}
It can be shown by scalarization for multicriterion optimization similarly to Proposition 3.2 of \cite{Cheng_PAMI2014}.
\end{proof}


{ In the following sections, we will explain how to optimize \cref{eq_DRMK} and how to perform classification with $w^\star$. 
	For clarity and smoothness of the organization of this paper, we will first discuss how to perform classification with the optimal $w^\star$ in \cref{sec_optimal_discriminant} and then provide the detailed optimization in \cref{sec_DRM}. }

\section{discriminative ridge regression-Based Classification}
\label{sec_optimal_discriminant}
Having estimated from \cref{eq_DRMK} the vector of optimal combining coefficients, denoted by ${w}^{\star}$,   
we use it to determine the category of $x$. 
In this paper, we mainly consider the case in which the $g$ groups are exhaustive of the whole instance space; 
in the non-exhaustive case, the corresponding minimax decision rule \cite{Cheng_PAMI2014} can be used analogously.   

In the exhaustive case, the decision function 
is built with the projection of $x$ onto the subspace spanned by each group. 
Specifically, the projection of $x$ onto the subspace of $C_i$ is $\psi_i := A {w}^{\star}|_{C_i}$, where 
${w}^{\star}|_{C_i}$ represents restricting ${w}^{\star}$ to $C_i$ in that $({w}^{\star}|_{C_i})_j = {w}^{\star}_j \textbf{1}(x_j \in C_i)$, 
with $\textbf{1}(\cdot)$ being the indicator function. 
Similarly, denoting $\cup_{k=1, k \ne i}^g C_k$ by $\bar{C}_i$, 
the projection of $x$ onto the subspace of $\bar{C}_i$ is $\bar{\psi}_i := A {w}^{\star}|_{\bar{C}_i}$. 
Now define 
\begin{align}
\nonumber \delta_i :=  || x - \psi_i||_2^2 + || \bar{\psi}_i||_2^2, 
\end{align} 
which measures the dis-similarity between $x$ and the examples in class $C_i$. 
Then the decision rule chooses the class with the minimal dis-similarity. This rule has a clear geometrical interpretation 
and it works intuitively: When $x$ truly belongs to $C_k$, 
$\psi_k  = \sum_{i=1}^{n_k} x_i^{k} ({w}^{\star}|_{C_k})_i \approx x$, while $\bar{\psi}_k \approx 0$ 
due to the class separability properties imposed by the discriminative ridge regression, and thus, $\delta_k$ is approximately $0$; 
whereas for $j \ne k$, $\psi_j = \sum_{i=1}^{n_j} x_i^{j} ({w}^{\star}|_{C_k})_i \approx 0$, 
 while $\bar{\psi}_j \approx x$ as $\bar{C}_j \supseteq C_k$, and thus
${\delta}_j$ is approximately $2 ||x||_2^2$ . Hence the decision rule picks $C_k$ for $x$.  
In the kernel space the corresponding $\delta_i^{\phi}$ is derived in the minimax sense as follows \cite{Cheng_PAMI2014}, 
\begin{equation}
\label{eq_mindistrule}
\delta_i^{\phi} =  ({w}^{\star}|_{C_i})^T K {w}^{\star}|_{C_i} + ({w}^{\star}|_{\bar{C}_i})^T K {w}^{\star}|_{\bar{C}_i} - 2 ({w}^{\star}|_{C_i})^T K_x.
\end{equation}
And the corresponding decision rule is
\begin{equation}
\label{kmindist_decision}
\hat{i} = {\text{argmin}}_{i \in \{1, \cdots, g\} } \delta_i^{\phi}.
\end{equation}
{ It should be noted that the decision rule \cref{kmindist_decision} depends on the weighting coefficient vector $w^\star$, which accounts for discriminative information between classes by minimizing a discriminative ridge regression model \cref{eq_DRMK}. Thus, with $w^\star$, the dis-similarity $\delta_i$ or $\delta_i^{\phi}$ accounts for the discriminativeness of the classes and helps to classify the target example.} Now we obtain the discriminative ridge regression-based classification outlined in Algorithm \ref{CDRM}. 
Related parameters such as $\alpha$ and $\beta$, 
and a proper kernel can be chosen using the standard cross validation (CV).

\begin{algorithm}
\scriptsize
\caption{discriminative ridge regression-based classification with $\rho( ||w |C ||_{q, r} )$-regularization.} 
\label{CDRM}
\begin{algorithmic} [1] 
\STATE \textbf{Input:} Training examples $A=[x_1,...,x_n] \in
\mathcal{R}^{p\times n}$, class labels $y=[y_1, \cdots, y_n]$, 
with $y_k = m$ if $\sum_{i=1}^{m-1} n_i + 1 \le k \le \sum_{i=1}^{m} n_i$ for $m=1, \cdots, g$, 
a test example $x \in \mathcal{R}^p$, $\alpha \ge 0$, $\beta > 0$, and a selected kernel.
\STATE 
Compute $K$, $H$, $B$, and $K_x$ by \cref{def_K,def_HB,def_Kx}.
\STATE Solve the minimization problem \cref{eq_DRMK}.
\STATE Compute \cref{eq_mindistrule}.
\STATE Estimate the label for $x$ using \cref{kmindist_decision}.
\STATE \textbf{Output:} Estimated class label for $x$.
\normalsize
\end{algorithmic}
\end{algorithm}

How to solve optimization problem \cref{eq_DRMK} efficiently in Step 3 is critical to our algorithm. 
We mainly intend to use such $\rho(|| w | C ||_{q, r})$ that is convex in  $w$ 
to efficiently attain the global optimum of \cref{eq_DRMK}. 
Any standard optimization tool may be employed including, for example, CVX package \cite {Grant2011}.
In the following, we shall focus on a special case of $\rho(|| w | C ||_{q, r})$ 
for the sake of deriving efficient and scalable optimization algorithms for high-dimensional data and large-scale data; 
nonetheless, it is noted that the discriminative ridge regression-based classification is constructed for general $(q, r)$ and $\rho(\cdot)$.

\section{discriminative ridge machine}
\label{sec_DRM}
Discriminative ridge regression-based classification with a particular regularization of $(q, r) = (2, 2)$
will be considered, leading to the DRM. 

\subsection{Closed-Form Solution to DRM}
\label{sec_l22}
Let  $(q, r) = (2, 2)$ and $\rho(x) = \frac{1}{2} x^2$, then we have the regularization term
$\rho( || w | C ||_{q, r} ) = \frac{1}{2} || w | C ||_{2, 2}^2 = \frac{1}{2} ||w||_2^2$. 
The discriminative ridge  regressionproblem \cref{eq_DRMK} reduces to
\begin{align}
\label{eq_drr}
{\text{(DRM)}} \ \underset{ w \in {\mathcal{R}}^n}{\text{min}}  \frac{1}{2} w^T (K + \alpha (H - B) + \beta I) \ w -  w^T K_x, 
\end{align}
which is an unconstrained convex quadratic optimization problem leading to a closed-form solution,
\begin{align}
\label{eq_drr_solution}
	{w}^{\star} = (K + \alpha H - \alpha B + \beta I)^{-1} K_x, 
\end{align}
with $\alpha \ge 0$ and $\beta > 0$. Hereafter we always require $\beta>0$ to ensure the existence of 
the inverse matrix, because
$Q := K + \alpha H - \alpha B$ is p.s.d. and $Q+ \beta I$ is strictly positive definite. 
The minimization problem \cref{eq_drr} can be regarded as a generalization of kernel ridge regression. 
Indeed, when  $\alpha = 0$ kernel ridge regression is obtained. More interestingly, 
when $\alpha \ne 0$ the discriminative information is incorporated to go beyond the kernel ridge regression.  

For clarity, the DRM with the closed-form formula is outlined as Algorithm \ref{Alg_DRM}. 

\begin{remark}
1. The computational cost of the closed-form solution \cref{eq_drr_solution}
is $O(n^2 p + n^3)$, because calculating $K$ and $B$ from $A$ costs $O(n^2 p)$ while, in general,  
the matrix inversion  is $O(n^3)$ by using, for example, QR decomposition or Gaussian elimination. 
2. For high-dimensional data with $p \gg n$, the cost is $O(n^2 p)$ 
with Algorithm \ref{Alg_DRM}. 
This cost is scalable in $p$ when $n$ is moderate. 
3. For large-sample data with $n \gg p$, the complexity of using \cref{eq_drr_solution} is $O(n^3)$, 
which renders Algorithm \ref{Alg_DRM} impractical with big $n$. 
To have a more efficient and scalable, though possibly inexact, solution to \cref{eq_drr}, 
in the sequel we will construct three iterative algorithms suitable for large-scale data.  
\end{remark}

\subsection{DRM Algorithm Based on Gradient Descent (DRM-GD)}
\label{sec_bigdata}
For large-scale data with big  $n$ and $n \gg p$, the closed-form solution \cref{eq_drr_solution} 
may be computationally impractical in real applications, which calls for more efficient and scalable algorithms. 

Three new algorithms will be established  to alleviate this computational bottleneck. 
In light of the linear SVM on large-scale data classification, e.g., \cite{bottou2010large}, 
we will consider the use of the linear kernel in the DRM, which turns out to have provable linear cost.
The first algorithm relies on classic gradient descent optimization method, denoted by DRM-GD; 
the second hinges on an idea of proximal-point approximation (PPA) similar to \cite{ZhouChengTNNLS2014} to eliminate the matrix inversion; 
and the third uses a method of accelerated proximal gradient line search for theoretically proven quadratic convergence. 
This section presents the algorithm of DRM-GD, 
and the other two algorithms will be derived in subsequent sections. 

Denote by $f(w)$ the objective function of the DRM,
\begin{align}
\nonumber f(w) = \frac{1}{2} w^T Q \ w -  w^T K_x + \frac{\beta}{2} ||w||_2^2.
\end{align}
Its gradient is 
\begin{align}
\label{eq_gradient}
 \nabla_w f(w) = (Q + \beta I) w - K_x. 
\end{align}
Suppose we have obtained $w^{(t)}$ at iteration $t$. 
 At next iteration, the gradient descent direction is $- \nabla_w f(w^{(t)})$, in which the optimal step size is given  by the exact line search, 
\begin{align}
\label{eq_gradient_stepsize}
\nonumber d^{(t)} = {\text{argmin}}_{d \ge 0} f(w^{(t)} - d \ \nabla_w f(w^{(t)})) \\
= \frac{ (\nabla_w f(w^{(t)}))^T \nabla_w f(w^{(t)})}{(\nabla_w f(w^{(t)}))^T (Q+\beta I) \nabla_w f(w^{(t)})}. 
\end{align} 
Thus, the updating rule using gradient descent is
\begin{align}
\label{eq_gradient_updating}
	w^{(t+1)} = w^{(t)} - d^{(t)} \nabla_w f(w^{(t)}).
\end{align}

\begin{algorithm}
	\scriptsize
	\caption{DRM with the closed-form formula.} 
	\label{Alg_DRM}
	\begin{algorithmic} [1] 
		\STATE All are the same as Algorithm \ref{CDRM} except for Step 3.
		\STATE Step 3: Compute ${w}^{\star}$ with formula \cref{eq_drr_solution}.
		\normalsize
	\end{algorithmic}
\end{algorithm}

\begin{algorithm}[t]
	\scriptsize
	\caption{DRM-GD.} 
	\label{alg_gradient_descent}
	\begin{algorithmic} [1] 
		\STATE \textbf{Input:} Training examples $A=[x_1,...,x_n] \in
		\mathcal{R}^{p\times n}$, class labels $y=[y_1, \cdots, y_n]$, 
		$x \in \mathcal{R}^p$, $\alpha \ge 0$, $\beta > 0$, and a selected kernel.
		\STATE Initialize $t=0$ and $w^{(0)} \in {\mathcal{R}}^n$, and a convergence tolerance $\epsilon$.
		\STATE Repeat
		\STATE $\quad$Compute the gradient \cref{eq_gradient}.
		\STATE $\quad$Compute the step size \cref{eq_gradient_stepsize}.
		\STATE \label{GD_updating} $\quad$Compute $w^{(t+1)}$ with the updating rule \cref{eq_gradient_updating}.
		\STATE Until $|| w^{(t+1)} - w^{(t)} ||_2 \le \epsilon$, then exit iteration and let ${w^{\star}} = w^{(t+1)}$; else, let $t= t+1$.
		\STATE \label{GD_classifying} Compute  \cref{eq_mindistrule}.
		\STATE Estimate the label for $x$ using \cref{kmindist_decision}.
		\STATE \textbf{Output:} Estimated class label for $x$.
		\normalsize
	\end{algorithmic}
\end{algorithm}

The resulting DRM-GD is outlined in Algorithm \ref{alg_gradient_descent}.
Its main computational burden is on $Q$ which costs,
in general, $n^2 p$ floating-point operations (flops) \cite{Boyd2004} for a general kernel. 
This paper mainly considers multiplications for flops. 
After getting $Q$, $\nabla_w f(w^{(t)})$ is obtained by matrix-vector product with $n^2+n$ flops, and so is $d^{(t)}$ with $n^2 + 3n$ flops. 
Thus, the overall count of flops is about $(p+2) n^2$ with a general kernel. 

With the linear kernel the computational cost can be further reduced by exploiting its particular structure of $K = A^T A$. 
Given any $v \in {\mathcal{R}}^n$, $B^{i} v^{i}$ can be computed as $\frac{1}{n_i} (x^{i})^T (x^{i} v^{i})$ 
which costs $2 p n_i$ flops for any $1 \le i \le g$, 
and thus computing $B v = [(B^{1} v^{1})^T, \cdots, (B^{g} v^{g})^T]^T$ requires $2 p n$ flops. 
Similarly, $Kv = A^T (A v)$ takes $p n$ flops by matrix-vector product, 
since $Av = \sum_i x^i v^i$ can be readily obtained from $B v$ computation. 
As $H v = [v_1 (x_1^T x_1), \cdots, v_n (x_n^T x_n)]$, the total count of flops is $(p+1) n$ for getting $Hv$.  
Therefore, computing $Q v$ needs $(4p+1) n$ flops. 
Each of \cref{eq_gradient,eq_gradient_stepsize} requires $Q v$ and $\beta v$ types of computation, 
hence each iteration of updating $w^{(t)}$ to $w^{(t+1)}$ costs
$(8p +  7) n$ overall flops, including $3n$ additional ones for computing $d^{(t)}$ and $d^{(t)} \nabla_w f(w^{(t)})$. 

As a summary, the property of the DRM-GD, including the cost, is given in \cref{thm_gradient_property}. 
\begin{theorem}[\bf{Property of DRM-GD}]
\label{thm_gradient_property} 
With any $w^{(0)} \in {\mathcal{R}}^n$, $\{  w^{(t)}   \}_{t=0}^{\infty}$ generated by 
Algorithm \ref{alg_gradient_descent} 
converges in function values to the unique minimum of $f(w)$ with a linear rate. 
Each iteration costs $(p+2) n^2$ flops with a general kernel, in particular, $(8p+7)n$ flops with the linear kernel.
\end{theorem}

\noindent{\em{Proof}}: For the gradient descent method, the sequence of objective function values $\{ f(w^{(t)}) \}$ 
converges to an optimal value with a rate of $O(1/t)$ \cite{LevitinPolyak1966}. The convergence and its rate of $f(w)$ are thus standard \cite{Boyd2004}. 
As $f(w)$ is strongly convex, the minimum is unique.  $\hfill \Box$


\subsection{DRM Algorithm Based on Proximal Point Approximation}
\label{sec_majorization}
PPA is a general method for finding a zero of a maximal monotone operator and solving non-convex optimization problems \cite{rockafellar1976monotone}. 
Many algorithms have been shown to be its special cases, including the method of multipliers, the alternating direction method of multipliers, and so on. 
Our DRM optimization is a convex problem; nonetheless, we employ the idea of PPA in this paper for potential benefits of efficiency and scalability.  
At iteration $t$, having obtained the minimizer $w^{(t)}$, we construct an augmented function around it, 
\begin{align}
\nonumber F(w; w^{(t)}) := f(w) + \frac{1}{2} (w - w^{(t)})^T (c I -Q) (w - w^{(t)}), 
\end{align} 
where $c$ is a constant satisfying $c \ge \sigma_{max}(Q)$. 
We minimize $F(w, w^{(t)})$ to update the minimizer, 
\begin{align}
\nonumber  w^{(t+1)} := {\text{argmin}}_{w \in {\mathcal{R}}^n} F(w; w^{(t)}). 
\end{align}
As $F(w; w^{(t)})$ is a strongly convex, quadratic function, its minimizer is calculated 
 directly by setting its first-order derivative equal to zero,
\[
\nabla F(w; w^{(t)}) = Qw - K_x + \beta w -  Q(w-w^{(t)}) + c(w - w^{(t)}) = 0,
\] 
which results in 
\begin{align}
\label{eq_majorization}
	w^{(t+1)} = (K_x - Q w^{(t)} + c w^{(t)})/(\beta+c).
\end{align} 

The computational cost of each iteration is reduced compared to the closed-form solution \cref{eq_drr_solution}, 
which is formally stated in the following. 

\begin{prop}
With a general kernel, the updating rule \cref{eq_majorization} costs about $(p+2) n^2$ flops.
\end{prop}  

{\em{Proof}}: The cost of getting  $K$ is $p n^2$, since it has $n^2$ elements and each costs $p$ flops.  
Subsequently, $\alpha(H-B)$ needs about $n^2$ flops. Hence computing $Q$ needs about $(p+1) n^2$ flops. Evidently $Q w^{(t)}$ needs $n^2$ flops.  $\hfill \Box$

As a summary, the PPA-based DRM, called DRM-PPA, is outlined in Algorithm \ref{alg_DRM_Iterative}. 
Starting from any initial point $w^{(0)} \in {\mathcal{R}}^n$, repeatedly applying \cref{eq_majorization} in Algorithm \ref{alg_DRM_Iterative} generates a sequence of points
$\{  w^{(t)}   \}_{t=0}^{\infty}$. 
The property of this sequence will be analyzed subsequently. 

\begin{algorithm}
\scriptsize
\caption{DRM-PPA with any kernel.} 
\label{alg_DRM_Iterative}
\begin{algorithmic} [1] 
\STATE \textbf{Input:} Training examples $A=[x_1,...,x_n] \in
\mathcal{R}^{p\times n}$, class labels $y=[y_1, \cdots, y_n]$, 
$x \in \mathcal{R}^p$, $\alpha \ge 0$, $\beta >0$, and a kernel.
\STATE 
Compute $K$,  $H$, $B$, and $K_x$ using \cref{def_K,def_HB,def_Kx}. 
\STATE Initialize $t=0$ and $w^{(0)} \in {\mathcal{R}}^n$, set $c \ge \sigma_{max}(K + \alpha(H-B))$, and $\epsilon >0$.
\STATE Repeat
\STATE $\quad$ \label{PPA_updating} Compute $w^{(t+1)}$ with \cref{eq_majorization}.
\STATE Until $|| w^{(t+1)} - w^{(t)} ||_2 \le \epsilon$, then exit iteration and let ${w}^{\star} = w^{(t+1)}$; otherwise, let $t= t+1$. 
\STATE \label{PPA_classifying} Compute \cref{eq_mindistrule}.
\STATE Estimate the label for $x$ with \cref{kmindist_decision}.
\STATE \textbf{Output:} Estimated class label for $x$.
\normalsize
\end{algorithmic}
\end{algorithm}

\begin{theorem}[{\bf{Convergence and optimality of DRM-PPA}} ]
\label{PPA_theorem_convergence}
With any $w^{(0)} \in {\mathcal{R}}^n$, the sequence of points $\{  w^{(t)}   \}_{t=0}^{\infty}$ generated by 
Algorithm \ref{alg_DRM_Iterative} gives a monotonically non-increasing value sequence $\{ f(w^{(t)})   \}_{t=0}^{\infty}$ 
which converges to the globally minimal value of $f(w)$. Furthermore, the sequence 
$\{  w^{(t)}   \}_{t=0}^{\infty}$ itself converges to $w^{\star}$ in \cref{eq_drr_solution}, the unique minimizer of $f(w)$.
\end{theorem}

Before the proof, we point out two properties of $F(w; w^{(t)})$ which are immediate by the definition, 
\begin{align}
\label{major_property}
\nonumber F(w; w^{(t)}) &\ge f(w),  \quad \forall w; \\ 
F(w^{(t)}; w^{(t)}) &= f(w^{(t)}).
\end{align}

\noindent {\em{Proof of \cref{PPA_theorem_convergence}}}: By re-writing $f(w)$ as 
\[ 
\frac{1}{2} || (Q+\beta I)^{1/2} w - (Q+\beta I)^{- 1/2} K_x ||_2^2 - \frac{1}{2} K_x^T(Q+\beta I)^{-1} K_x, 
\]
we know $f(w)$ is lower bounded by
$- \frac{1}{2} K_x^T(Q+\beta I)^{-1} K_x$.
From the following chain of inequality it is clear that $\{ f(w^{(t)}) \}_{t=0}^{\infty}$ is a monotonically non-increasing sequence,  
\[
f(w^{(t)}) = F(w^{(t)}; w^{(t)}) \ge F(w^{(t+1)}; w^{(t)}) \ge f(w^{(t+1)}). 
\]
The first inequality of the chain holds by \cref{eq_majorization} since $w^{(t+1)}$ is the global minimizer of 
$F(w; w^{(t)})$, and the second by \cref{major_property}. 
Hence, $\{ f(w^{(t)}) \}_{t=0}^{\infty}$ converges.  

Next we will further prove that $\lim_{t \rightarrow \infty}  w^{(t)}$ exists and it is the unique minimizer of $f(w)$. 
By using the following equality
\begin{align}
\nonumber w^{(t+1)} - {w}^{\star} &= (K_x - Q w^{(t)} + c w^{(t)})/(\beta + c) - {w}^{\star} \\
\nonumber & = ( (Q + \beta I) {w}^{\star}  - Q w^{(t)} + c w^{(t)})/(\beta + c) - {w}^{\star} \\
\nonumber & = (c I - Q) (w^{(t)} - {w}^{\star})/(\beta + c),
\end{align}
where the first equality holds by \cref{eq_majorization} and the second by \cref{eq_drr_solution}, we have
\begin{align} 
\nonumber ||w^{(t+1)} - {w}^{\star}||_2 
\nonumber & \le ||(c I - Q)||_2 \ ||w^{(t)} - {w}^{\star}||_2 /(\beta + c) \\
\nonumber & = \frac{c - \sigma_{min}(Q)}{c+\beta} \ ||w^{(t)} - {w}^{\star}||_2	\\
\label{eq_upper_bound}
&	\le ||w^{(0)} - {w}^{\star}||_2 \ \left(\frac{c - \sigma_{min}(Q)}{c+\beta}\right)^{t+1}.
\end{align}
Here, the first inequality holds by the definition of the spectral norm, 
and the subsequent equality holds because $||c I - Q||_2 = \sigma_{max} (c I - Q) = c - \sigma_{min}(Q)$.
Because $\beta >0$ and $Q$ is p.s.d., we have $\sigma_{min}(Q) \ge 0$,   
and $ \frac{c - \sigma_{min}(Q)}{c+\beta} < 1$. Hence, as $t \rightarrow \infty$, 
$||w^{(t+1)} - {w}^{\star}||_2 \rightarrow 0$ for any $w^{(0)}$; that is, $w^{(t+1)} \rightarrow {w}^{\star}$. 
The uniqueness is because of the strict convexity of $f(w)$. 
 $\hfill \Box$

As a consequence of the above proof, the convergence rate is also readily obtained. 

\begin{theorem}[\bf Convergence rate of DRM-PPA]
\label{thm_rate_DRM} The convergence rate of $\{  w^{(t)}   \}_{t=0}^{\infty}$ generated by Algorithm \ref{alg_DRM_Iterative} is at least linear. Given a convergence threshold $\epsilon > 0$, the number of iterations is upper bounded by $\log \frac{||w^{(0)} - w^{\star}||_2}{\epsilon} / \log(\frac{c+ \beta}{c-\sigma_{min}(Q)})$, which is approximately 
$\frac{c-\sigma_{min}(Q)}{\sigma_{min}(Q) + \beta} \log \frac{||w^{(0)} - w^{\star}||_2}{\epsilon} $ when $\frac{\sigma_{min}(Q)+\beta}{c - \sigma_{min}(Q)}$ is small.
\end{theorem}

\noindent{\em{Proof}}: Because
\begin{align}
\nonumber \underset{t \rightarrow \infty}{lim} sup \frac{||w^{(t+1)} - {w}^{\star} ||_2}{||w^{(t)} - {w}^{\star} ||_2}
\le \frac{c- \sigma_{min}(Q)}{c+\beta} < 1, 
\end{align}
the convergence rate is at least linear. 
For a given $\epsilon$, the maximal number of iterations satisfies $||w^{(0)} - {w}^{\star}||_2 \ \left(\frac{c - \sigma_{min}(Q)}{c+\beta}\right)^{t} \le \epsilon$ 
by \cref{eq_upper_bound}, hence the upper bound in the theorem. When $\frac{\sigma_{min}(Q)+\beta}{c - \sigma_{min}(Q)}$ is small, $\log(\frac{c+ \beta}{c-\sigma_{min}(Q)})$ is approximately $\frac{\sigma_{min}(Q)+ \beta}{c-\sigma_{min}(Q)}$ by using $\log(1+x) \approx x$ for small $x$. $\hfill \Box$

\noindent{\bf{Remarks}} 1. The convergence rate of the DRM-PPA  
depends on $\log \frac{||w^{(0)} - w^{\star}||_2}{\epsilon}$, which is determined by the distance between the initial and optimal points, and the final accuracy.
The rate also depends on a factor of 
$\frac{c-\sigma_{min}(Q)}{\sigma_{min}(Q) + \beta}$. 
The smaller this factor, the faster the convergence. 
Using $c$ closer to  $\sigma_{max}(Q)$ 
renders the DRM-PPA faster.\  
2. With large $n$ yet small $p$, $\sigma_{max}(Q)$ can be obtained efficiently as shown in next section.
When both $n$ and $p$ are large, 
we may use a loose upper bound 
\begin{align}
\nonumber \! & \sigma_{max}(Q) = ||Q||_2 \le ||K||_2 + \alpha || H- B ||_2 \\
\nonumber \! &\le \min \{ Tr (K), || K ||_{\infty} \} + \alpha || H ||_2 
\end{align}
\begin{align}
\nonumber \! &=  \min \{ \sum_{i=1}^{n} k(x_i, x_i), \max_i \{ \sum_{j=1}^{n} |k(x_i, x_j)| \} \} 	\\
	& \quad +  \alpha \max_i \{ k(x_i, x_i) \}.
\end{align} 
Here, the first inequality is by the triangle inequality, and the second by the Gershgorin Theorem. 
Usually this upper bound is much larger than $\sigma_{max}(Q)$ and the convergence of the DRM-PPA would be slower.
Alternatively, the implicitly restarted Lanczos method (IRLM) by Sorensen \cite{sorensen92}
can be used to find the largest eigenvalue. The IRLM relies only on matrix-vector product and can compute an approximation of 
$\sigma_{max}(Q)$ as $c$ with any specified accuracy tolerance at little cost. As yet another alternative, 
the backtracking method may be used to iteratively find a 
proper $c$ value as in Algorithm \ref{alg-APG} given in \cref{sec_APG}.

\subsection{Scalable Linear-DRM-PPA for Large-Scale Data}

To further reduce the complexity, now we consider using the linear kernel in the DRM-PPA. 
With a special structure of $K = A^T A$,  we exploit the matrix-vector multiplication to improve the efficiency. 
Given a vector $v \in {\mathcal{R}}^n$,  
since $B$ is block diagonal and $B v = [ (B^{1} v^{1})^T, \cdots, (B^{g} v^{g})^T ]^T$, 
we need only to compute each block separately to obtain $Bv$, 
\begin{align}
\label{eq_Bw_linearkernel}
 u^{i} \leftarrow x^{i} v^{i}, \quad
B^{i} v^{i} \leftarrow \frac{1}{n_i}(x^{i})^T u^{i}, \quad \ 1 \le i \le g. 
\end{align}
Consequently, the computation of $B v$ costs $(2p+1) n$ flops.
The computation of $Kv$ is given by 
\begin{align}
\label{eq_Kw}
 u \leftarrow Av = \sum_{i=1}^g x^i v^i, \quad
Kv \leftarrow A^T u, 
\end{align}
which only needs $p n$ flops.
Finally, $Hv$ is computed as
\begin{align}
\label{eq_Hw}
Hv = [ (x_1^T x_1) v_1, \cdots, (x_n^T x_n) v_n ]^T,
\end{align}
which takes $(p + 1) n$ flops. 
Overall, it takes $(4p +2) n$ flops to get $Qv$, and thus each iteration of \cref{eq_majorization} costs $(4p+4)n$ flops. 
For clarity, the linear-DRM-PPA is outlined in Algorithm \ref{alg_scalable_DRM}.

\begin{algorithm}
\scriptsize
\caption{Linear-DRM-PPA.} 
\label{alg_scalable_DRM}
\begin{algorithmic} [1] 
\STATE \textbf{Input:} Training examples $A=[x_1,...,x_n] \in
\mathcal{R}^{p\times n}$, class labels $y=[y_1, \cdots, y_n]$, 
a test example $x \in \mathcal{R}^p$, $\alpha \ge 0$ and $\beta >0$. 
\STATE  Initialize $t=0$, and an arbitrary $w^{(0)} \in {\mathcal{R}}^n$. Set $\epsilon > 0$ and $c \ge \sigma_{max}(Q)$.
\STATE Repeat
\STATE \quad Let $v = w^{(t)}$.
\STATE \quad \label{linearPPA_updating}  Compute $Kv$, $Bv$, and $Hv$ using \cref{eq_Kw,eq_Bw_linearkernel,eq_Hw}. 
\STATE \quad Update $w^{(t+1)}$ using $Q w^{(t)}= K v + \alpha H v - \alpha B v$, and \cref{eq_majorization}.
\STATE Until $|| w^{(t+1)} - w^{(t)} ||_2 \le \epsilon$, then exit the iteration and let ${w}^{\star} = w^{(t+1)}$; otherwise let $t= t+1$. 
\STATE Compute \cref{eq_mindistrule}.
\STATE Estimate the label with \cref{kmindist_decision}.
\STATE \textbf{Output:} Estimated class label for $x$.
\normalsize
\end{algorithmic}
\end{algorithm}

\noindent{\bf{Remarks}} In Step 2 $\sigma_{max}(Q)$ needs to be estimated. With big $n$ and small $p$, 
$\sigma_{max}(K)$ may be computed as $\sigma_{max}(A^T A) = \sigma_{max}(A A^T) = \sigma_{max} (\sum_{i=1}^n x_i x_i^T)$, with a cost of $O(p^2 n)$; while $\sigma_{max}(H) = \max_i {x_i^T x_i}$, taking $p n$ flops. Thus, $\sigma_{max}(Q) \le \sigma_{max}(K) + \alpha \sigma_{max}(H)$, taking $O((p^2+p)n)$ flops.

By using the matrix-vector product Algorithm \ref{alg_scalable_DRM} has a linear computational cost in $n$ when $p$ is small, 
which is summarized as follows. 

\begin{prop}
\label{prop_majorization} 
The cost of Algorithm \ref{alg_scalable_DRM} is $(4p+4) m n$ flops,  
by taking $(4p+4) n$ flops per iteration for $m$ iterations needed to converge whose upper bound is given by \cref{thm_rate_DRM}.
\end{prop}

\subsection{DRM Algorithm Based on Accelerated Proximal Gradient}	
\label{sec_APG}
Accelerated proximal gradient (APG) techniques, originally due to Nesterov
\cite{Nesterov1983} for smooth convex functions, have been developed with a convergence rate of $O(t^{-2})$ in function values \cite{BeckTeboulle2009}. 
For its potential in efficiency for large-scale data,
we exploit the APG idea here to build an algorithm for the DRM, denoted by DRM-APG.  
It does not compute the exact step size in the gradient descent direction, 
but rather uses a proximal approximation of the gradient.  
Thus the DRM-APG can be regarded as a combination of PPA and gradient descent methods for potentially faster convergence. 

Evidently $f(w)$ has Lipschitz continuous gradient because 
$|| \nabla f(w) - \nabla f(v)||_2 \le || Q+ \beta I ||_2 ||w-v||_2 \le b ||w -v ||_2$ for any $w, v \in {\mathcal{R}}^n$, 
where $b$ is an upper bound for $|| Q + \beta I||_2$, for example, $b = c + \beta$, with $c$ defined in \cref{sec_majorization}.  First, let us consider the case in which the Lipschitz constant $b$ is known. 
Around a given $v$ a PPA to $f(w)$ is built, 
\begin{align}
	G(w | v, b) = f(v) + (w-v)^T \nabla f(v) + \frac{b}{2} ||w -v ||_2^2.
\end{align}
Note that $G(w|  v, b)$ is an upper bound of $f(w)$, and both functions have identical values and first-order derivatives at $w = v$. 
We minimize $G(w|  v, b)$ to get an updated $w$ based on $v$, 
\begin{align}
\nonumber
	\hat{w}_b(v) =: {\text{argmin}}_w \ G(w|  v, b) = v - \frac{1}{b} \nabla f(v). 
\end{align}

Next, we will consider the case in which $b$ is  unknown. In practice, especially for large-scale data, the Lipschitz constant might not be readily available. 
In this case, a backtracking strategy is often used. Based on its current value, $b$ is iteratively increased until  
$G(w|  v, b)$ becomes an upper bound of $f(w)$ at the point $\hat{w}_b(v)$; that is, $f( \hat{w}_b(v) ) \le G( \hat{w}_b(v) |  v, b)$. 
By straightforward algebra this is equivalent to 
\begin{align}
\label{backtracking_condition}
	\nabla f(v) ^T Q \nabla f(v) \le (b- \beta) || \nabla f(v) ||_2^2.
\end{align}
Thus, more specifically, the strategy is that a constant factor $\eta > 1$
will be repeatedly multiplied to $b$ until \cref{backtracking_condition} is met if the current $b$ value fails to satisfy \cref{backtracking_condition}.

In iteration $t+1$, after obtaining a proper $b^{(t+1)}$ with the backtracking, 
the update rule of $w^{(t+1)}$ is 
\begin{align}
\label{eq_APG_w}
\nonumber	w^{(t+1)} &= \hat{w}_{b^{(t+1)}}(v^{(t)}) = v^{(t)} - \frac{1}{b^{(t+1)}} \nabla f(v^{(t)}) \\
	&= v^{(t)} - \frac{1}{b^{(t+1)}} ((Q +\beta I)v^{(t)} - K_x). 
\end{align}
The stepsize $d^{(t+1)}$ and the auxiliary $v^{(t+1)}$ are updated by,
\begin{align}
\label{eq_APG_d_v}
\nonumber	d^{(t+1)} &= (1+\sqrt{1+4 (d^{(t)})^2})/2, \\
	v^{(t+1)} &= w^{(t+1)} + \frac{d^{(t)} - 1}{d^{(t+1)}} (w^{(t+1)} - w^{(t)}).
\end{align}

In summary, the procedure of the DRM-APG is outlined in Algorithm \ref{alg-APG}. Matrix-vector multiplication may
be employed in Step \ref{step-check-condition} which costs, assuming $\nabla f(w_k)$ has been obtained,  $i_t (n+1)^2$ flops with a general kernel, 
and $i_t (2(p+1)n + p(g+1))$ flops with the linear kernel. Here, the integer $i_t$ is always finite because of the Lipschitz condition on the gradient of $f(w)$. 
Steps \ref{step-update-w} and \ref{step-update-d-v} cost $n^2$  flops with a general kernel, while in the particular case of the linear kernel, 
$(4p+5)n$ flops. The convergence property and cost are stated in \cref{thm_APG_property}. 

\begin{table}
\tiny
\centering
\caption{Summary of Small to Medium Sized Data Sets Used in Experiments}
\label{Tab:summary_small}%
\resizebox{0.4\textwidth}{!}{%
\begin{threeparttable}
\begin{tabular}{l l l l l l l l}
\Xhline{1.2pt}
Type 		& Data Set 		& Dim. 		& Size  	& Class & Training 	& Testing 	& Notes		\\ \Xhline{1.2pt}
High Dim. 	& GCM  			& 16,036	& 189		& 14	& 150 		& 39   		& multi-class, small  size		 \\
			& Lung cancer	& 12,533 	& 181  		& 2 	& 137 		& 44 		& binary-class, small  size		\\
 			& Sun data		& 54,613	& 180  		& 4 	& 137 		& 43		& multi-class, small  size	 \\
 			& Prostate		& 12,600	& 136  		& 2 	& 103	 	& 33		& binary-class, small size 		\\ 			
 			& Ramaswamy		& 16,063	& 198  		& 14 	& 154 		& 44		& multi-class, small size 		\\
 			& NCI			& 9,712		& 60  		& 9 	& 47 		& 13		& multi-class, small size 		\\
 			& Nakayama		& 22,283	& 112  		& 10 	& 84 		& 28		& multi-class, small size 	
 			\vspace{0.1cm}\\

Image data  & EYaleB		& 1,024		& 2,414  	& 38	& 1,814 	& 600		& multi-class, moderate size 	 \\
        	& AR			& 1,024		& 1,300  	& 50	& 1,050	 	& 250		& multi-class, moderate size 	 \\
        	& PIX			& 10,000	& 100		& 10	& 80		& 20		& multi-class, small size		\\
        	& Jaffe			& 676		& 213		& 10	& 172		& 41		& multi-class, small size		\\
        	& Yale			& 1,024		& 165		& 15	& 135		& 30		& multi-class, small size		
        	\vspace{0.1cm}\\
        	 			
Low Dim. 	& Optical pen	& 64		& 1,797  	& 10	& 1,352	 	& 445		& multi-class, moderate size  \\
			& Pen digits	& 16		& 10,992  	& 10	& 1,575	 	& 9,417		& multi-class, moderate size 	\\
			& Semeion		& 256		& 1593		& 10	& 1,279		& 314		& multi-class, moderate size  	\\
			& Iris			& 4			& 150  		& 3		& 114 	 	& 36		& multi-class, small size 		\\
 			& Wine			& 13		& 178  		& 3		& 135	 	& 43		& multi-class, small size 	 \\
 			& Tic-tac-toe	& 9			& 958  		& 2		& 719	 	& 239		& binary-class, moderate size 	\\

\Xhline{1.2pt} 
\end{tabular} 
\end{threeparttable}}
\end{table}

\begin{theorem}[{\bf Property of DRM-APG} ]
 \label{thm_APG_property} With any $w^{(0)} \in {\mathcal{R}}^n$, 
 $\{  w^{(t)}   \}_{t=0}^{\infty}$ generated by 
Algorithm \ref{alg-APG} converges in function value to the global minimum of $f(w)$ with a rate of $O(\frac{1}{t^2})$. 
Each iteration costs $O(p n^2)$ flops with a general kernel, while  $O(p n)$ flops with the linear kernel.
\end{theorem} 

\noindent{\em{Proof}}: The convergence and its rate are standard; 
see, for example, \cite{Nesterov1983}. The uniqueness of the global minimizer is due to the strong convexity of $f(w)$. 
The cost is analyzed similarly to \cref{prop_majorization} or \cref{thm_gradient_property}. $\hfill \Box$

\noindent{\bf{Remarks}} 1. The backtracking step \ref{step-check-condition} 
of Algorithm \ref{alg-APG} can be skipped
when the Lipschitz constant $b$ is pre-computed. 2. With the linear kernel, 
each iteration of the APG algorithm takes $(4p+5)n$ flops. 
Consequently, linear cost is achieved in a way similar to Algorithm \ref{alg_gradient_descent} or \ref{alg_scalable_DRM}.
The detail of the linear DRM-APG is not repeated here.

The desirable properties of linear efficiency and scalability 
in Algorithms \ref{alg_gradient_descent}, \ref{alg_scalable_DRM} and \ref{alg-APG}
can be generalized to other kernels, provided some conditions such as low-rank approximation are met. 
One way for generalization is stated in \cref{prop-generalization}.

\begin{prop}[\bf Generalization of Linear DRM Algorithms to Any Kernel]
\label{prop-generalization} 
Algorithms  \ref{alg_gradient_descent},   \ref{alg_scalable_DRM}, 
and \ref{alg-APG} are applicable to any kernel function $k(\cdot, \cdot)$, provided its corresponding kernel matrix $K$ can be approximately factorized, 
$K \approx G^T G$, with $G \in {\mathcal{R}}^{r \times n}$ and  $r \ll n$. The computational cost of the generalized algorithm is $O(rmn)$, with $m$ iterations.
\end{prop}

\noindent {\em{Proof}}: Replacing $A$ by $G$ and, correspondingly, $p$ by $r$ in \cref{eq_Kw,eq_Bw_linearkernel,eq_Hw}, the conclusion follows for \algref{alg_scalable_DRM}; and so do Algorithms \ref{alg_gradient_descent} or \ref{alg-APG} similarly. $\hfill \Box$

\noindent {\bf{Remarks}} 1. The low-rank approximation of the kernel matrix has been used in \cite{ferris2004semismooth} \cite{fine2002efficient} 
to develop SVM algorithms of quadratic convergence rate. 
2. The  decomposition of $K$, in general, has a high cost of  $O(n^3)$, and thus is not always readily satisfied. 
In this paper we mainly consider using the linear kernel for scalable classification on large-scale data.

\begin{algorithm}
	\scriptsize
	\caption{DRM-APG with optional backtracking.} 
	\label{alg-APG}
	\begin{algorithmic} [1] 
		\STATE \textbf{Input:} $A=[x_1,...,x_n] \in
		\mathcal{R}^{p\times n}$, $y=[y_1, \cdots, y_n]$, 
		$x \in \mathcal{R}^p$, $\alpha \ge 0$, $\beta >0$, and a kernel.
		\STATE Initialize $t=0$, $w^{(0)} = v^{(0)} \in {\mathcal{R}}^n$, $b^{(0)} >0$, $\eta > 1, d^{(0)} = 1$,  and set $\epsilon>0$.
		\STATE Repeat
		\STATE $\quad$ \label{step-check-condition} (optional) Find the smallest nonnegative integer $i_t$ such that \\ \quad \quad \quad \quad \quad \cref{backtracking_condition} holds with $b := \eta^{i_t} b^{(t)}$. Let $b^{(t+1)}= b$.
		\STATE $\quad$ \label{step-update-w} Update $w^{(t+1)}$ by \cref{eq_APG_w}.
		\STATE $\quad$ \label{step-update-d-v} Update $d^{(t+1)}$ and $v^{(t+1)}$ by \cref{eq_APG_d_v}.
		\STATE Until $|| w^{(t+1)} - w^{(t)} ||_2 \le \epsilon$, then exit the iteration and let ${w^{\star}} = w^{(t+1)}$; else, let $t= t+1$.
		\STATE \label{APG_classifying} Compute \cref{eq_mindistrule}.
		\STATE Estimate the label for $x$ with \cref{kmindist_decision}.
		\STATE \textbf{Output:} Estimated class label for $x$.
		\normalsize
	\end{algorithmic}
\end{algorithm}

\begin{table*}
	\tiny
	\centering
	\caption{ Classification Accuracy of SVM, KNN, RF, C4.5, NB, and DRM on High- and Low-Dimensional Data }
	\resizebox{0.8\textwidth}{!}{%
		\begin{threeparttable}
			\begin{tabular}{l l c c c c c c c c c}
				\Xhline{1.2pt}
				Type & Dataset 	& SVM(R) 			& SVM(P) 					& KNN(E) 			& KNN(C) 			& RF 				& C4.5				& NB					& DRM(R) 				& DRM(P)  \\ \Xhline{1.2pt}
				High Dim.
				
				& Nakayama	 	& 0.6286$\pm$0.0621	& 0.7619$\pm$0.1010			& 0.7429$\pm$0.0722	& 0.7524$\pm$0.0782	& 0.6762$\pm$0.0916	& 0.3619$\pm$0.1043	&  -----------------	& 0.7619$\pm$0.0583		& \bf{0.7905$\pm$0.0543}	\\ 
				& Leukemia 		& 0.7500$\pm$0.0000	& \bf{1.0000$\pm$0.0000}	& 0.9500$\pm$0.1118	& 0.9750$\pm$0.0559	& 0.9250$\pm$0.1118	& 0.8500$\pm$0.2054	& \bf{1.0000$\pm$0.0000}& \bf{1.0000$\pm$0.0000}& \bf{1.0000$\pm$0.0000}	\\ 
				& Lungcancer 	& 0.9864$\pm$0.0203	& 0.9909$\pm$0.0124		 	& 0.9955$\pm$0.0102	& 0.9955$\pm$0.0102	& 0.9955$\pm$0.0102	& 0.9273$\pm$0.0466	& 0.9864$\pm$0.0124		& \bf{1.0000$\pm$0.0000}& \bf{1.0000$\pm$0.0000}	\\ 
				& Prostate 		& 0.6788$\pm$0.0507	& \bf{0.9455$\pm$0.0498}	& 0.8061$\pm$0.0819	& 0.8121$\pm$0.0498	& 0.8727$\pm$0.0657	& 0.6545$\pm$0.0628	& 0.6061$\pm$0.0851 	& 0.8848$\pm$0.0449		& 0.9333$\pm$0.0542	\\ 
				& GCM	 		& 0.4974$\pm$0.0344	& 0.6462$\pm$0.0380			& 0.6154$\pm$0.0480 & 0.6462$\pm$0.0380	& 0.6051$\pm$0.0466	& 0.2410$\pm$0.0562	& 0.7077$\pm$0.0229		& 0.8308$\pm$0.0389		& \bf{0.8410$\pm$0.0493}	\\ 
				& Sun Data		& 0.5535$\pm$0.0477	& 0.6884$\pm$0.0746			& 0.7302$\pm$0.0535	& 0.7070$\pm$0.0628	& 0.7163$\pm$0.0104	& 0.4884$\pm$0.0465	& 0.6977$\pm$0.0658		& 0.7442$\pm$0.0637		& \bf{0.7488$\pm$0.0666}	\\ 
				& Ramaswamy		& 0.6182$\pm$0.0743	& 0.6906$\pm$0.0380			& 0.7000$\pm$0.0296 & 0.7000$\pm$0.0296 & 0.6909$\pm$0.0903	& 0.5273$\pm$0.0943	& 0.6227$\pm$0.0674		& \bf{0.8273$\pm$0.0729}& 0.8000$\pm$0.0588	\\
				& NCI			& 0.0769$\pm$0.0000	& 0.5385$\pm$0.0769			& 0.4308$\pm$0.1397 & 0.4154$\pm$0.1287	& 0.4462$\pm$0.1668	& -----------------	& -----------------	& \bf{0.6308$\pm$0.1141}& 0.5231$\pm$0.1264	\vspace{0.1cm} 	\\

				Image Data
				& AR			& 0.6320$\pm$0.0210		& 0.9464$\pm$0.0112	& 0.6552$\pm$0.0411	& 0.6816$\pm$0.0346	& 0.7944$\pm$0.0384 		& 0.1040$\pm$0.0216	& -----------------	& \bf{0.9856$\pm$0.0112}	& 0.9776$\pm$0.0112	\\ 
				& EYaleB		& 0.8163$\pm$0.0146		& 0.9233$\pm$0.0210	& 0.7490$\pm$0.0103	& 0.8507$\pm$0.0126	& 0.9700$\pm$0.0047 		& 0.1370$\pm$0.0117	& -----------------	& 0.9830$\pm$0.0040			&\bf {0.9920$\pm$0.0059}	\\ 
				& Jaffe			& \bf{1.0000$\pm$0.0000}& 0.4146$\pm$0.0000	& 0.9951$\pm$0.0109	& 0.9951$\pm$0.0109	& \bf{1.0000$\pm$0.0000}	& 0.8098$\pm$0.0966	& -----------------	& 0.9951$\pm$0.0109			& \bf{1.0000$\pm$0.0000}	\\ 
				& Yale			& 0.7133$\pm$0.0960		& 0.7067$\pm$0.1188	& 0.6467$\pm$0.0691	& 0.6533$\pm$0.0989	& 0.7733$\pm$0.0435 		& 0.4067$\pm$0.0925	& 0.6267$\pm$0.0596	& 0.8067$\pm$0.0596			& \bf{0.8133$\pm$0.0558}	\\ 
				& PIX			& 0.9400$\pm$0.0584		& 0.9300$\pm$0.0477	& 0.9600$\pm$0.0652 & 0.9600$\pm$0.0652	& \bf{0.9800$\pm$0.0274} 	& 0.8537$\pm$0.0827	& 0.8900$\pm$0.0962 & 0.9600$\pm$0.0652			& 0.9600$\pm$0.0652	\vspace{0.1cm} \\ 
				
				Low Dim.
				& Semeion			& 0.9433$\pm$0.0042		& 0.9204$\pm$0.0050	& 0.9255$\pm$0.0138 & 0.9185$\pm$0.0036		& 0.9236$\pm$0.0172 	& 0.6287$\pm$0.0486	&  -----------------	& \bf{0.9624$\pm$0.0116}	& 0.9611$\pm$0.0079	\\ 
				& Pendigits			& 0.9867$\pm$0.0020		& 0.9865$\pm$0.0017	& 0.9835$\pm$0.0019	& 0.9839$\pm$0.0019		& 0.9686$\pm$0.0031 	& 0.6963$\pm$0.0278	&  -----------------	& \bf{0.9911$\pm$0.0017}	& 0.9910$\pm$0.0015	\\ 
				& Opticalpen		& \bf{0.9924$\pm$0.0034}& 0.9856$\pm$0.0020	& 0.9888$\pm$0.0036	& 0.9879$\pm$0.0059		& 0.9730$\pm$0.0086 	& 0.1011$\pm$0.0000	&  -----------------	& 0.9915$\pm$0.0040			& \bf{0.9924$\pm$0.0030}	\\ 
				
				& Iris				& 0.9778$\pm$0.0232		& 0.9778$\pm$0.0232 & 0.9667$\pm$0.0304	& \bf{0.9833$\pm$0.0152}& 0.9556$\pm$0.0317 	& 0.9611$\pm$0.0317	& 0.9596$\pm$0.0421	& 0.9667$\pm$0.0562		& \bf{0.9833$\pm$0.0152}	\\ 
				& Wine				& 0.9674$\pm$0.0353		& 0.9814$\pm$0.0195	& 0.8791$\pm$0.0602	& 0.9349$\pm$0.0382		& \bf{0.9860$\pm$0.0208}& 0.9395$\pm$0.0353	& 0.9767$\pm$0.0233	& 0.9116$\pm$0.0504		& 0.9581$\pm$0.0382	\\ 
				& Tic$\_$tac$\_$toe	& 0.8494$\pm$0.0118		& 0.9724$\pm$0.0128	& 0.8418$\pm$0.0174	& 0.8427$\pm$0.0158		& 0.9264$\pm$0.0134 	& 0.7665$\pm$0.0321	& 0.6996$\pm$0.0130	& \bf{0.9950$\pm$0.0035}& 0.9941$\pm$0.0092	\\ 
				
				\Xhline{1.2pt}
				\label{tab_acc_small}
			\end{tabular}
		\end{threeparttable}}
		
		\scriptsize For each data set, the best accuracy rate is boldfaced. The performance is represented as average accuracy $\pm$ standard deviation.
	\end{table*}

\section{Experiments}
\label{sec_experiments}
In this section, we present how our method performs on standard test data sets in real world applications. 
First, we show our results on high- and low-dimensional data and compare them with several state-of-the-art methods. 
Then we compare the linear DRM and the linear SVM on large-scale data and evaluate our three algorithms DRM-PPA, DRM-GD and DRM-APG 
in terms of accuracy rate, time cost for training, and convergence rate. Last, we testify the DRM in imbalanced data classification.

All experiments are implemented in Matlab on a 4-core Intel Core i7-4510U 2.00GHz laptop with 16G memory. 
We terminate the algorithm when $|| w^{(t+1)} - w^{(t)} ||_2 \le 10^{-5}$ for the DRM algorithms, or when they reach the maximal number of iterations of 150. 
As experimentally shown, all three algorithms usually converge within a few tens of iterations.  

\subsection{Application to Real World Data}
\label{sec:app_small}

We conduct experiments on 19 data sets in three categories: high-dimensional data for gene expression classification, image data for face recognition, and low-dimensional data including hand-written digits and others. Among them, the first two types are high-diensional while the last is low-dimensional. Their characteristics and the partition of training and testing sets are listed in \cref{Tab:summary_small}. { For each data set, we perform experiments on 5 random trials and report their average accuracy rates and standard deviations for different algorithms. }

\subsubsection{Gene Expression}
\label{sec_geneexpression}

For high-dimensional data, we have eight gene expression data sets: Global Cancer Map (GCM), 
Lung caner, Sun data \cite{sun2006neuronal}, 
Prostate, Ramaswamy, Leukemia, NCI, and Nakayama \cite{nakayama2007gene}. 
{
	To illustrate the effectiveness of the proposed method, we compare it with currently state-of-the-art classifiers on such data, including the SVM, KNN, RF, C4.5, and Naive Bayes (NB).}
For each of these classifiers, leave-one-out cross validation (LOOCV) is conducted for parameter selection. For the SVM, we use linear, polynomial, and radial basis function (rbf) kernel. 
The order of polynomial kernel is selected from the set of $\{2,3,4,5,8,10\}$, and the variance parameter of the rbf kernel is selected from $\{0.001,0.01,0.1,1,10,100,1000\}$. 
{The balancing parameter of SVM is selected from the set of $\{0.001,0.01,0.1,1,10,100,1000\}$. }
For the KNN, we adopt the Euclidean and cosine based distances and the number of neighbors are selected within $\{2,3,4,5,6,7,8\}$.
For RF, the number of trees is selected from the set $\{10,20,30,40\}$. 
{ For the C4.5, we choose the percentage of incorrectly assigned examples at a node to be $\{5\%,6\%,7\%,8\%,9\%,10\%\}$.}
For the DRM, the parameter options for kernels remain the same as the SVM and the value of $\alpha$ and $\beta$ range in $\{0.001,0.01,0.1,1,10,100,1000\}$.
All these parameters are determined by the LOOCV. We conduct experiments on both scaled and unscaled data for all classifiers and report the better performance. Here, for each data set, the scaled data set is obtained by dividing each feature value with the infinite norm of the corresponding feature vector. We report the average performance on the 5 splits in \cref{tab_acc_small}. 
{ Here, for clearer illustration of how these methods perform, we seperately report the best performances of rbf and polynomial kernels for SVM and DRM, as well as those of cosine and Euclidean distances for KNN. It should be noted that in practice the kernel and distance types can be determined by cross-validation. }

From \cref{tab_acc_small}, it is seen that the proposed method has the highest accuracy on seven out of eight data sets, except for Prostate, with significant improvement in classification accuracy. For example, the DRM has at least 20\%, 13\% and 10\% improvements on GCM, Ramaswamy, and NCI data sets, respectively. Besides, the DRM achieves the second best performance, which is compariable to the best. Moreover, in many cases, the DRM obtains the top two highest classification accuracy. These observations have confirmed the effectiveness of the DRM on high-dimensional data and application of gene expression classification.

\begin{figure}[!tb]
	\centering{
\includegraphics[width=0.9\columnwidth]{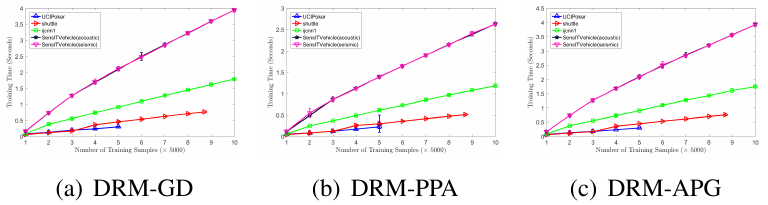}
		\caption{Training time of an example versus training sample size $n$, for different data sets. All algorithms run 50 iterations with $\alpha=10^{-3}$ and $\beta=10^4$.}
		\label{fig_time_all}
	}
\end{figure}

\subsubsection{Face Recognition}
Face recognition is an important classification problem and we examine the performance of the DRM on high-dimensional data in this application. In this test, five commonly used face data sets are used, including Extended Yale B \cite{georghiades2001few}, AR \cite{martinez1998ar}, Jaffe, Yale \cite{belhumeur1997eigenfaces}, and PIX \cite{hond1997distinctive}. 
The EYaleB data set has 2,414 images for 38 individuals under different viewing conditions. All original images are of size $192\times 168$ whose cropped and resized version of $32\times32$ pixels \cite{georghiades2001few} are used in the experiments. 
For AR data set, 25 women and 25 men are uniformly selected. All images have size of $32\times32$ pixels. 
Jaffe data set contains face images of 10 Japanese females with 7 different facial expressions. Each image has size of $26\times 26$.
PIX data set collects 100 gray scale images of $100\times100$ pixels from 10 objects. 
Yale data set contains 165 gray scale images of 15 persons with 11 images of size 32$\times$32 per person. 
In the experiments, all images are vectorized as columns in the corresponding data matrix. 

{
It should be pointed out that deep convolutional neural network (CNN) \cite{Krizhevsky2012ImageNet} has achieved promising performance for many large-scale image classification tasks.
However, in this paper, we do not include CNN as a baseline method due to the following reasons:
1) The key reason is the lack of comprehensive theoretical understanding of learning with deep neural networks such as the generalization ability, 
and thus they are often criticized for being used as a ``black box'' \cite{alain2016understanding};
2) CNN is treated as a specialized classifier for image data, whereas the DRM is a general classifier that can be used in various scenarios. 
It is unfair to compare such methods;
3) It is convincing to claim at least comparable performance of the DRM to CNN
\footnote{
As reported in later parts of this section, we test the DRM on 8 image data sets, including 5 face image data and 3 digit image data.
The DRM has achieved higher than 96\% accuracy on 7 data sets, out of which 5 data sets have accuracy above 98\% and 4 data sets above 99\%.
Thus, it is convincing that the DRM is competitive even if CNN has better performance since there is little room to improve the performance.}. 
}
In this test, all experimental settings remain the same as in \cref{sec_geneexpression} and we report the average performance on 5 splits in \cref{tab_acc_small}. 
From the results, it is seen that the proposed method achieves the best performance on 4 out of 5 data sets with significant improvements. 
On PIX data, the DRM is not the best, but is comparable to the RF and obtains the second best performance. 

In summary, the DRM achieves the highest accuracy rates on eleven out of thirteen high-dimensional data sets, including gene expression and face image data sets, and comparable performance on the rest two. Even though the other methods may achieve the best performance on a few data sets, in many cases they are not competetive. Meanwhile, the DRM always show competitive performance, which implies the effectiveness of the DRM.

\subsubsection{Low-Dimensional Data}
Though the DRM is manily intended on classifying high-dimensional data, we also test its classification capability on classic low-dioensional data. For this test, we use six data sets, inluding three hand-written digits and three others. 
Hand-written digits are image data sets of much fewer pixels than face images and thus treated as low-dimensional data sets. Three widely used hand-written data sets \cite{Bache+Lichman:2013} are tested: hand-written pen digits, optical pen, and Semeion Handwritten Digit (SHD) \cite{semeion2014}. They all have 10 classes representing the digits 0-9. 
The sizes of images in Optical pen, digits, and SHD are $8 \times 8$, $4 \times 4$, and $16 \times 16$, respectively. All images are reshaped into vectors. 
Besides, three data sets, iris, wine and tic-tac-toe, from the UCI Machine Learning Repository, are included here to show the applicability of the DRM to classic, low-dimensional data. 
All experimental settings remain the same as previous tests and we report the average performance on five splits in \cref{tab_acc_small}.

\cref{tab_acc_small} indicates that the DRM performs the best on five data sets. Especially, the DRM achieves the best performance on all these hand-written digits data sets. Moreover, the DRM has comparable performance using both rbf and polynomial kernels. These observations verify that the DRM is also effective on low-dimensional data classification.

\begin{figure}[!tb]
	\centering
\includegraphics[width=0.9\columnwidth]{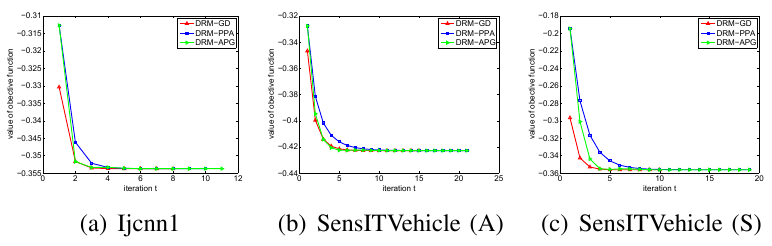}
	\caption{Value of objective function versus iteration, for the DRM on different data sets.}
	\label{fig_conv_comp}
\end{figure}

\begin{figure}[!tb]
	\centering{
		\includegraphics[width=0.9\columnwidth]{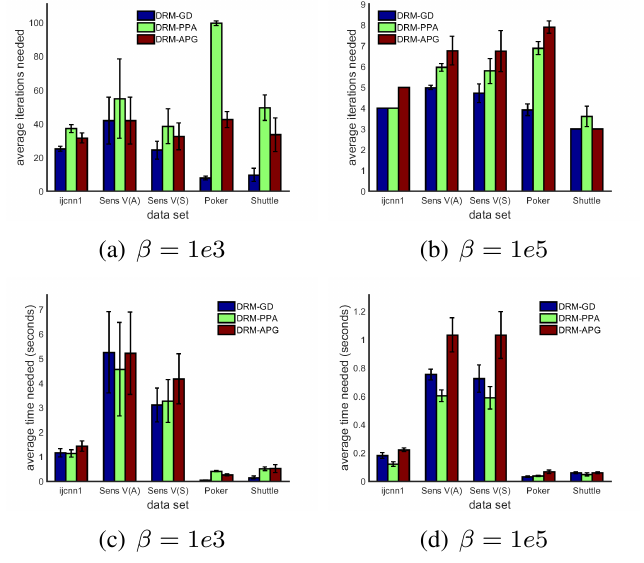}
		\caption{ Iterations and time needed by three DRM algorithms on different data sets, with $\epsilon = 1e-5, \alpha = 1e-3$. 
			SensV(A) and SensV(S) stand for SensITVehicle (Acoustic) and (Seismic), respectively.}
		\label{fig_iterations4datasets}
	}
\end{figure} 

\begin{figure}[!tb]
	\centering{
		\includegraphics[width=0.9\columnwidth]{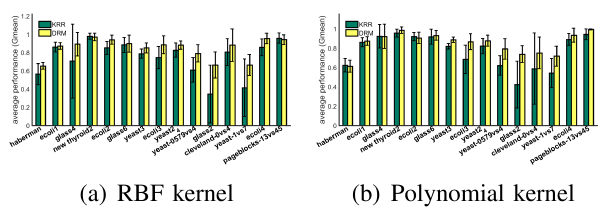}
		\caption{ Comparison between DRM and KRR, where all parameters are selected by CV.}
		\label{fig_abalation_overall}
	}
\end{figure}


\subsubsection{Large-Scale Data}
We evaluate the performance of the linear-DRM-PPA, linear-DRM-GD, and linear-DRM-APG on 5 standard large-scale data sets: {UCI Poker} \cite{Bache+Lichman:2013}, {ijcnn1} \cite{chang2001ijcnn}, {SensIT Vehicle (Acoustic)} \cite{duarte2004vehicle}, {SensIT Vehicle (Seismic)} \cite{duarte2004vehicle}, and {Shuttle}.
%
%
{
All data sets have already been split into training and test sets.
It is noted that such large sample size would effectively avoid the effects of occasionality of samples and it is unnecessary to re-split the data.
Thus we follow the original splitting.}
We conduct random split CV on the training sets to select the parameters $\alpha$ and $\beta$ with a grid search from $10^{-6}$ to $10^{-1}$ and from $10^{-6}$ to $10^{8}$, respectively. To reduce computational complexity, we randomly select a subset from the training set to do the CV. The size of the subset is usually set to 3000. All data sets are scaled to $\left[-1, 1\right]$ by dividing each feature with the infinite norm of the corresponding feature vector. 
We compare the prediction accuracy with the $L_{1}$-SVM and $L_{2}$-SVM for which TRON \cite{lin2008trust} 
is employed as the solver. \cref{Tab:acc_bigdata3} gives the accuracy rates, from which two observations are made: 
(1) The accuracy of the SVM is (almost) exactly the same as that reported in \cite{chang2008coordinate}; (2) The proposed DRM method outperforms the SVM on all the data sets significantly; especially, on Shuttle the DRM exceeds the SVM by $53.2\%$.


\begin{figure*}[!tb]
	\centering{
		\includegraphics[width=1.6\columnwidth]{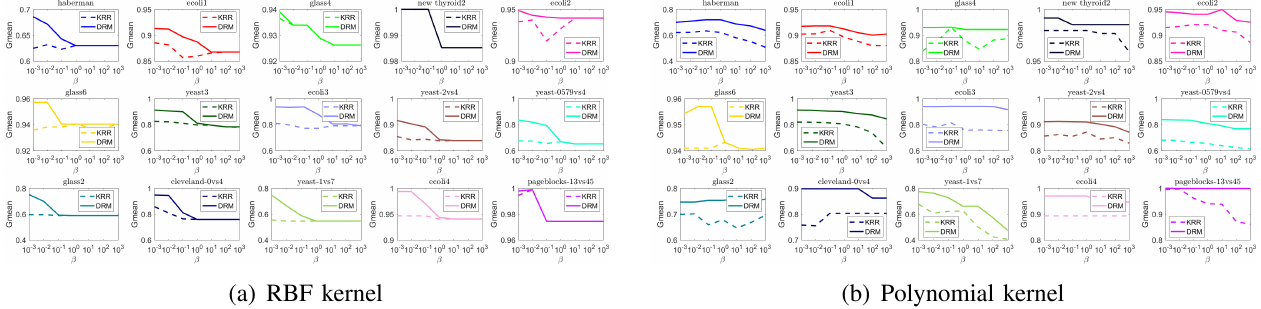}
		\caption{ Alalation study: comparison between DRM and KRR by fixing $\beta$ value and varing the others.}
		\label{fig_abalation_alpha}
	}
\end{figure*}

\begin{table}[!tb]
\centering
\resizebox{0.4\textwidth}{!}{%
\begin{threeparttable}
\caption{Large-Scale Data Used in Experiments}
\label{Tab:summary_bigdata}
\begin{tabular}{l l l l l l l l}
\Xhline{1.8pt}
Dataset 					& Size  					&\quad	& Class  	&\quad	& Training 	&	\qquad& Testing 		\\ \Xhline{1.2pt}
 UCI Poker 					& 1,025,010 $\times$ 10 	&\quad	& 10		&\quad	& 25,010 	&	\qquad& 1,000,000 	\\
  ijcnn1 					& 141,691 	$\times$ 22 	&\quad	& 2 		&\quad	& 49,990 	&	\qquad& 91,701 			\\
  SensIT Vehicle (A) 		& 98,528 	$\times$ 50 	&\quad	& 3 		&\quad	& 78,823 	&	\qquad& 19,705  		\\
  SensIT Vehicle (S) 		& 98,528 	$\times$ 50 	&\quad	& 3 		&\quad	& 78,823  	&	\qquad& 19,705 			\\
   Shuttle 					& 58,000 	$\times$ 9 		&\quad	& 7 		&\quad	& 43,500 	&	\qquad& 14,500  		\\
\Xhline{1.8pt}
\end{tabular}
\footnotesize For the five data sets, we follow the original splitting.
\end{threeparttable}
}
\end{table}

\begin{table}[!tb]
\centering
\resizebox{0.4\textwidth}{!}{%
\begin{threeparttable}
\caption{ Prediction Accuracy by Using SVM and DRM}
\label{Tab:acc_bigdata3}
\begin{tabular}{l c c c c c}
\Xhline{1.8pt}
Dataset 			& $L_{1}$-SVM 	& $L_{2}$-SVM 	& DRM-GD & DRM-PPA  & DRM-APG \\ \Xhline{1.2pt}
UCI Poker 			& 31.5\%  		& 33.8\%		& \textbf{50.1\%} &	\textbf{50.1\%} & \textbf{50.1\%}\\
IJCNN1				& 67.3\%		& 74.2\%		& \textbf{90.1\%} &	\textbf{90.5\%}	& \textbf{90.5\%}\\
SensIT Vehicle (A) 	& 43.5\%		& 45.9\%		& \textbf{65.1\%} &	\textbf{64.6\%}	& \textbf{65.5\%}\\
SensIT Vehicle (S) 	& 41.6\%		& 42.2\%		& \textbf{62.8\%} &	\textbf{62.8\%}	& \textbf{62.8\%}\\
Shuttle 			& 35.9\%		& 29.7\%		& \textbf{89.6\%} &	\textbf{89.2\%}	& \textbf{90.0\%} \\
\Xhline{1.8pt}
\end{tabular}
\end{threeparttable}
}
\end{table}

\begin{table}
	\centering
	\caption{Confusion Matrix}
	\label{Tab_confusion_matrix}%
	\resizebox{0.4\textwidth}{!}{%
		\begin{threeparttable}
			\begin{tabular}{l c c}
				\Xhline{1.8pt}
				& Predicted as positive 		& Predicted as negative	\\ \Xhline{1.2pt}
				Actually positive	& True Positive (TP)					& False Negative (FN)			\\	
				Actually negative	& False Positive (FP)					& True Negative	 (TN)		\\		
				\Xhline{1.2pt} \\
			\end{tabular} 
		\end{threeparttable}}
	\end{table}
	
	\begin{table}
		\centering
		\caption{Summary of Imbalanced Data Sets}
		\label{tab_summary_imbalanced}%
		\resizebox{0.4\textwidth}{!}{%
			\begin{threeparttable}
				\begin{tabular}{ l l l l || l l l l }
					\Xhline{1.8pt}
					Data Set 					& Dim. 		& Size  	& IR 		& Data Set 					& Dim. 		& Size  	& IR 		\\ \Xhline{1.2pt}
					
					haberman					& 3			& 306		& 2.78		& yeast-2vs4			& 8			& 514	  	& 9.08	 	\\
					ecoli1						& 7			& 336  		& 3.36 		& yeast-05679vs4	& 8			& 528	  	& 9.35	 	\\		
					glass4						& 9			& 214  		& 3.98 		& glass2						& 9			& 214  		& 11.59		\\	
					new thyroid 2				& 5			& 215  		& 5.14	  	& cleveland-0vs4		& 13 		& 173  		& 12.31 	\\ 
					ecoli2						& 7			& 336  		& 5.46	 	& yeast-1vs7			& 7			& 459	  	& 14.30	 	\\
					glass6						& 9			& 214  		& 6.38	 	& ecoli4						& 7			& 336  		& 15.80 	\\
					yeast3						& 8			& 1484  	& 8.10	 	& page-blocks-13vs45	& 10		& 472  		& 15.86		 \\	
					ecoli3						& 7			& 336  		& 8.60	 	& & & & \\
					
					\Xhline{1.2pt} \\
				\end{tabular} 
			\end{threeparttable}}
		\end{table}
\begin{table*}[!t]
	\centering
	\tiny
	\caption{ Classification Performance (G-mean) of SVM, KNN, RF, C4.5, NB, SLDA, and DRM on Imbalanced Data }
	\label{tab_imbalanced}%
	\resizebox{0.8\textwidth}{!}{%
		\begin{threeparttable}
			\begin{tabular}{l cccccccc cc}
				\Xhline{1.8pt}
				Data Set 			& SVM(R)				& SVM(P)  			& KNN(E) 				& KNN(C)				& RF				& C4.5					& NB					& SLDA					&	DRM(R)				& DRM(P) 				\\ \Xhline{1.2pt}
				haberman			& 0.4212$\pm$0.1451		& 0.5578$\pm$0.1249	& 0.5315$\pm$0.0455		& 0.5188$\pm$0.1318		& 0.4532$\pm$0.0761	& --------------------	& 0.5000$\pm$0.1562		& 0.4359$\pm$0.1048		& \bf{0.6541$\pm$0.0364}& 0.6129$\pm$0.0618		\\
				ecoli1				& \bf{0.8830$\pm$0.0383}& 0.4082$\pm$0.2618 & 0.8651$\pm$0.0426 	& 0.8435$\pm$0.0830 	& 0.8517$\pm$0.0777	& 0.4093$\pm$0.0312		& --------------------	& 0.7792$\pm$0.0413	& 0.8714$\pm$0.0366		& 0.8728$\pm$0.0466 		\\		
				glass4				& 0.6849$\pm$0.1855		& 0.8858$\pm$0.1311 & 0.8479$\pm$0.1973		& 0.8479$\pm$0.1973		& 0.6358$\pm$0.3858	& --------------------	& 0.7210$\pm$0.1623		& 0.3528$\pm$0.3364	& 0.8926$\pm$0.1290		& \bf{0.9219$\pm$0.1257}  	\\	
				new thyroid 2		& 0.8936$\pm$0.0442		& 0.9824$\pm$0.0322 & \bf{0.9852$\pm$0.0332}& 0.9796$\pm$0.0308		& 0.9543$\pm$0.0511	& 0.3596$\pm$0.0192		& 0.9859$\pm$0.0142		& 0.5227$\pm$0.0880	& 0.9703$\pm$0.0406		& \bf{0.9852$\pm$0.0332} 	\\ 
				ecoli2				& 0.8143$\pm$0.0438		& 0.0000$\pm$0.0000	& 0.9380$\pm$0.0494		& \bf{0.9494$\pm$0.0398}& 0.8592$\pm$0.0715	& 0.3468$\pm$0.0495		& --------------------	& 0.5880$\pm$0.1993	& 0.9415$\pm$0.0510		& 0.9057$\pm$0.0573			\\
				glass6				& 0.8176$\pm$0.0836		& 0.8628$\pm$0.0797 & 0.8625$\pm$0.0766 	& 0.8843$\pm$0.0747		& 0.8673$\pm$0.0759	& --------------------	& 0.8837$\pm$0.0919		& 0.8503$\pm$0.0984	& 0.8995$\pm$0.0940		& \bf{0.9323$\pm$0.0496} 	\\
				yeast3				& 0.7253$\pm$0.0508		& -------------------- 	& 0.8141$\pm$0.0224 & 0.7752$\pm$0.0398		& 0.8222$\pm$0.0423	& 0.2129$\pm$0.1447		& --------------------	& 0.8510$\pm$0.0271	& 0.8509$\pm$0.0539		& \bf{0.8860$\pm$0.0279} 	\\
				ecoli3				& 0.3610$\pm$0.2159		& --------------------	& 0.7797$\pm$0.1493	& 0.7657$\pm$0.1438		& 0.6923$\pm$0.1161	& 0.2597$\pm$0.0361		& --------------------	& 0.6632$\pm$0.1561	& \bf{0.8862$\pm$0.0983}& 0.8668$\pm$0.0822			\\	
				yeast-2vs4			& 0.6542$\pm$0.1782		& -------------------- 	& 0.8322$\pm$0.0779 & 0.7710$\pm$0.0886		& 0.8222$\pm$0.1121	& 0.0685$\pm$0.1553		& --------------------	& 0.5558$\pm$0.0665	& \bf{0.8822$\pm$0.0459}& 0.8765$\pm$0.0582 		\\
				yeast-0579vs4		& 0.1258$\pm$0.1723		& --------------------	& 0.5686$\pm$0.1330	& 0.6375$\pm$0.1204		& 0.5840$\pm$0.1850	& 0.0557$\pm$0.0142		& --------------------	& 0.4578$\pm$0.0768	& 0.7911$\pm$0.0945		& \bf{0.7915$\pm$0.1058}	\\					
				
				glass2				& -------------------- 	& 0.3517$\pm$0.3378		& 0.5663$\pm$0.0736	& 0.3996$\pm$0.3816	& 0.1140$\pm$0.2550		& --------------------	& 0.4677$\pm$0.2643		& 0.3901$\pm$0.2497	& 0.6652$\pm$0.1441 	& \bf{0.7370$\pm$0.0869}	\\	
				cleveland-0vs4		& 0.5093$\pm$0.2909		& 0.5011$\pm$0.4858 	& 0.6332$\pm$0.3805 & 0.5294$\pm$0.3125	& 0.6103$\pm$0.3742		& 0.2880$\pm$0.0720		& 0.6962$\pm$0.4245		& 0.5596$\pm$0.0461	& \bf{0.8814$\pm$0.1783}& 0.7503$\pm$0.1635  		\\
				yeast-1vs7			& --------------------	& -------------------- 	& 0.5482$\pm$0.1487 & 0.4397$\pm$0.2643	& 0.4612$\pm$0.2578		& --------------------	& --------------------	& 0.4759$\pm$0.2709	& 0.6633$\pm$0.1161		& \bf{0.7174$\pm$0.1029} 	\\			
				ecoli4				& 0.8025$\pm$0.0870 	& -------------------- 	& 0.9464$\pm$0.0734 & 0.8610$\pm$0.1038 & 0.8435$\pm$0.2033 	& 0.2196$\pm$0.0494		& --------------------	& 0.8955$\pm$0.0694	& \bf{0.9545$\pm$0.0573}& 0.9343$\pm$0.0723 		\\ 	
				page-blocks-1-3vs45	& 0.6642$\pm$0.1955		& 0.9782$\pm$0.0395		& 0.9755$\pm$0.0456	& 0.9352$\pm$0.0592	& \bf{0.9989$\pm$0.0025}& 0.2539$\pm$0.0225		& 0.7324$\pm$0.1346		& 0.6589$\pm$0.0492	& 0.9433$\pm$0.0530		& 0.9932$\pm$0.0062			\\	
				
				\Xhline{1.2pt} \\
			\end{tabular} 
		\end{threeparttable}}
		\\ \scriptsize For each data set, the best G-mean value is boldfaced. The performance is represented as average G-mean $\pm$ standard deviation.
	\end{table*}
\begin{table*}
	\centering
	\tiny
	\caption{ Classification Performance (G-mean) of DRM and Sampling + Benchmarking Classifiers }
	\label{tab_imbalanced_sampling}%
	\resizebox{0.8\textwidth}{!}{%
		\begin{threeparttable}
			\begin{tabular}{l cccccccccc}
				\Xhline{1.8pt}
				Data Set 			& SMOTE+SVM			& BorSMOTE+SVM 		& SL-SMOTE+SVM  	& SMOTE+KNN			& BorSMOTE+KNN 		& SL-SMOTE+KNN		& SMOTE+RF			& BorSMOTE+RF		& SL-SMOTE+RF		& DRM					\\ \Xhline{1.2pt}
				haberman			& 0.6435$\pm$0.0727	& ----------------- & -----------------	& 0.5598$\pm$0.0447	& 0.5856$\pm$0.0568 & 0.5871$\pm$0.0230	& 0.5870$\pm$0.0770	& 0.5697$\pm$0.0519	& 0.5550$\pm$0.0376	& \bf{0.6541$\pm$0.0364}		\\
				ecoli1				& 0.8735$\pm$0.0399	& 0.8813$\pm$0.0101 & 0.8757$\pm$0.0367	& 0.8220$\pm$0.0905	& 0.8233$\pm$0.0667 & 0.8648$\pm$0.0878	& 0.8803$\pm$0.0443	& 0.8778$\pm$0.0544	& \bf{0.8904$\pm$0.0708}	& 0.8728$\pm$0.0466		\\		
				glass4				& 0.9003$\pm$0.1236	& 0.9003$\pm$0.1236 & 0.8883$\pm$0.1334	& \bf{0.9313$\pm$0.1257}	& \bf{0.9313$\pm$0.1257} & \bf{0.9313$\pm$0.1257}	& 0.7924$\pm$0.4431	& 0.8956$\pm$0.1353	& 0.8492$\pm$0.1951	& 0.9219$\pm$0.1257	 	\\	
				new thyroid 2		& 0.9824$\pm$0.0322	& 0.9824$\pm$0.0322 & 0.9944$\pm$0.0077	& \bf{1.0000$\pm$0.0000}	& \bf{1.0000$\pm$0.0000} & 0.9944$\pm$0.0077	& 0.9569$\pm$0.0487	& 0.9329$\pm$0.0379	& 0.9661$\pm$0.0534	& 0.9852$\pm$0.0332		\\ 
				ecoli2				& 0.8973$\pm$0.0544	& 0.8464$\pm$0.0744 & 0.9035$\pm$0.0630	& 0.9056$\pm$0.0681	& 0.8957$\pm$0.0724 & 0.9327$\pm$0.0462	& 0.8924$\pm$0.0660	& 0.8770$\pm$0.0408	& 0.9107$\pm$0.0408	& \bf{0.9415$\pm$0.0510}		\\
				glass6				& 0.9269$\pm$0.0511	& 0.8965$\pm$0.0656 & 0.8958$\pm$0.0601	& 0.9129$\pm$0.0739	& 0.9271$\pm$0.0738 & 0.9207$\pm$0.0752	& 0.9226$\pm$0.0362	& 0.9350$\pm$0.0543	& 0.8893$\pm$0.0767	& \bf{0.9329$\pm$0.0379}		\\
				yeast3				& \bf{0.9003$\pm$0.0362}	& 0.8714$\pm$0.0225 & 0.8965$\pm$0.0244	& 0.8401$\pm$0.0167	& 0.8228$\pm$0.0140 & 0.8532$\pm$0.0333	& 0.8824$\pm$0.0311	& 0.8876$\pm$0.0178	& 0.8758$\pm$0.0352	& 0.8860$\pm$0.0279		\\
				ecoli3				& \bf{0.8866$\pm$0.0205}	& 0.8701$\pm$0.0271 & 0.8829$\pm$0.0220	& 0.8081$\pm$0.0881	& ----------------- & 0.8134$\pm$0.0690	& 0.7970$\pm$0.1716	& 0.7641$\pm$0.1122	& 0.8483$\pm$0.0645	& 0.8862$\pm$0.0983		\\	
				yeast-2vs4			& \bf{0.8905$\pm$0.0362}	& 0.8405$\pm$0.0787 & 0.8877$\pm$0.0416	& -----------------	& 0.7573$\pm$0.1194 & 0.8641$\pm$0.0680	& 0.8544$\pm$0.0649	& 0.8781$\pm$0.0414	& 0.8575$\pm$0.0987	& 0.8822$\pm$0.0459		\\
				yeast-0579vs4		& 0.7936$\pm$0.0283	& \bf{0.8057$\pm$0.0495} & 0.7928$\pm$0.0825	& 0.7280$\pm$0.0421	& 0.7016$\pm$0.0774 & 0.7210$\pm$0.1273	& 0.7685$\pm$0.0703	& 0.7106$\pm$0.1253	& 0.7201$\pm$0.0964	& 0.7915$\pm$0.1058		\\					
				
				glass2				& 0.7311$\pm$0.1618	& 0.7036$\pm$0.1243 & 0.5421$\pm$0.0275	& 0.6183$\pm$0.1520	& 0.6178$\pm$0.1506 & 0.5996$\pm$0.1267	& 0.4924$\pm$0.2857	& 0.4580$\pm$0.2698	& 0.2460$\pm$0.3431	& \bf{0.7370$\pm$0.0869}		\\	
				cleveland-0vs4		& 0.9190$\pm$0.0649	& \bf{0.9451$\pm$0.0255} & 0.9158$\pm$0.0641	& 0.8040$\pm$0.1791	& 0.7585$\pm$0.2039 & 0.8063$\pm$0.1770	& 0.7048$\pm$0.1737	& 0.7660$\pm$0.2063	& 0.6278$\pm$0.4077	& 0.8814$\pm$0.1783		\\
				yeast-1vs7			& 0.7081$\pm$0.0785	& \bf{0.7452$\pm$0.0708} & 0.6865$\pm$0.0910	& 0.5666$\pm$0.0600	& 0.6002$\pm$0.0706 & 0.5321$\pm$0.0708	& 0.5249$\pm$0.1285	& 0.5556$\pm$0.0924	& 0.4127$\pm$0.2599	& 0.7174$\pm$0.1029		\\			
				ecoli4				& 0.9544$\pm$0.0233	& 0.9361$\pm$0.0233 & \bf{0.9709$\pm$0.0222}	& 0.9375$\pm$0.0750	& 0.9375$\pm$0.0750 & 0.9391$\pm$0.0765	& 0.9137$\pm$0.0719	& 0.9085$\pm$0.1266	& 0.9167$\pm$0.0725	& 0.9545$\pm$0.0573		\\ 	
				page-blocks-1-3vs45	& 0.9782$\pm$0.0395	& 0.9782$\pm$0.0395 & 0.9782$\pm$0.0395	& 0.9955$\pm$0.0074	& 0.9744$\pm$0.0452 & 0.9955$\pm$0.0074	& \bf{1.0000$\pm$0.0000}	& \bf{1.0000$\pm$0.0000}	& 0.9977$\pm$0.0051	& 0.9932$\pm$0.0062		\\					
				\Xhline{1.2pt} \\
			\end{tabular} 
		\end{threeparttable}}
		\\ \scriptsize For each data set, the best G-mean value is boldfaced. The performance is represented as average G-mean $\pm$ standard deviation.
	\end{table*}

As analyzed in \cref{sec_DRM}, all the three algorithms cost $O(mpn)$ flops. When $p$ is small, all of them have a cost of $O(m n)$ flops; 
nonetheless, they have different 
constant factors. To empirically demonstrate the computational cost, 
\cref{fig_time_all} plots the time cost versus training sample size $n$ for all three algorithms, 
{
	where we run experiments for 100 examples and record the average time as well as standard deviation. 
Here, we fix the iteration number to be 50, which guarantees the convergence of the algorithms as will be seen in later section.}
We observe that, with the same convergence tolerance, the DRM-APG has the largest time cost while the DRM-PPA the smallest on the same data set.  
\cref{fig_time_all} verifies that the average training time of the three algorithms is linear in $n$. 
It is noted that the proposed algorithms are naturally suitable for parallel computing because they treat one test example at a time.

%
%

To numerically verify the theoretical convergence properties of our algorithms, 
we  plot the value of objective function versus iteration in \cref{fig_conv_comp}. 
We use all training examples and the same testing example for the three algorithms, with $\alpha$ and $\beta$ fixed to be $10^{-3}$ and $10^{4}$, respectively. 
\cref{fig_conv_comp} indicates that the DRM-GD and DRM-PPA converge with the least and largest numbers of iterations, respectively, on most data sets; however, 
each iteration of the DRM-GD is more time consuming than that of the DRM-PPA. 
{
We have theoretically analyzed that larger $\beta$ ensures faster convergence for DRM-PPA.
As shown in \cref{fig_iterations4datasets}, similar conclusions can also be drawn for DRM-GD and DRM-APG.  
Especially, these three algorithms converge within a few iterations on all data sets with a large $\beta$.
We also report the time for convergence in \cref{fig_iterations4datasets}.
It is seen that the proposed algorithms converge fast and need only a few seconds, which implies potentials in real world applications.
}

{
	\textbf{Remark:} From the above extensive experiments, it is seen that the DRM has promising performance on various types of data as well as in various applications. Possible explanations about why the DRM performs well are as follows:
	Methods such as SVM perform training and testing stages separately, where they train a model with training set to predict the label of testing examples. It should be noted that a single model is used in the prediction for all testing examples, for each of whom the model has no discriminativeness. Meanwhile, the DRM performs training and testing in a seamlessly integrated way, where both training and testing examples are involved to train the model. Thus, the information of testing example is taken into account in the training stage and thus is revealed in the trained model, which helps to improve the prediction accuracy. Moreover, with testing examples as input, models are trained independently for each of them, which renders the models contain discriminative information specialized for the corresponding target examples.

}

\subsection{Application to Imbalanced Data}
Imbalanced data are often observed in real world applications. Typically, inimbalanced data the classes are not equally distributed. In this experiment, we evaluate the proposed method in classification imbalanced data. For the test, we have used 15 data sets from KEEL database, where the key statistics of them are provided in \cref{tab_summary_imbalanced}. Here, imbalance ratio (IR) measures the number of examples from large and small classes, where larger value indicates more imbalance. We adopt G-mean for evaluation, which is often used  in the literature \cite{su2007evaluation,barua2014mwmote}. For the definition of G-mean, we refer to the confusion matrix in \cref{Tab_confusion_matrix} and define 
\begin{itemize}
\item True Positive Rate = $TP_{rate}$ = TP / (TP + FN),
\item True Negative Rate = $TN_{rate}$ = TN / (TN + FP),
\item False Positive Rate = $FP_{rate}$ = FP / (TN + FP),
\item False Negative Rate = $FN_{rate}$ = FN / (TP + FN).
\end{itemize}
Then G-mean is defined as
\begin{equation}
\begin{aligned}
\text{$G$-$mean$}	&	= \sqrt{\frac{TP}{TP+FN} \times \frac{TN}{TN+FP}}.
\end{aligned}
\end{equation}

These data sets were intended for evaluation of classifiers on imbalanced data and originally splited into five folds. 
{ In our experiments, we follow the original splits to conduct 5 trials and report the average performance.
We include one more classifier in this test, i.e., sparse linear discriminant analysis (SLDA) \cite{clemmensen2011sparse}.
We do not include this method in previous experiments because it is over lengthy to train on high-dimensional data. 
For its weight on the elastic-net regression, we choose the values from $\{0.001,0.01,0.1,1,10,100,1000\}$ by LOOCV.
The reduced dimension is determined by the LOOCV from the set $\{1,3,5,7,9\}$.} 
Except for the change of evaluation metric, all the other settings remain the same as in \cref{sec_geneexpression}. 
We report the average performance on these five splits in \cref{tab_imbalanced}. 
From the results, it is seen that the DRM obtains the best performance on 12 out of 15 data sets. 
The SVM, KNN, and RF obtain one best performance, respectively. 
It is noted that in the cases where the DRM is not the best, its performance is still quite competive. 
For example, on page-blocks-1-3vs4 data set, the DRM has a G-mean value of 0.9932, which is slightly inferior to RF by about 0.005. In contrast, when the SVM, KNN, and RF are not the best, in many cases, their performance is far from being comparable to the DRM. 
These observations suggest the effectiveness of the DRM in imbalanced data classification.

{
Moreover, to further illustrate the effectiveness of the DRM, we also report the results of baseline classifiers with up-sampling technique including synthetic minority over-sampling technique (SMOTE) \cite{chawla2002smote} and its variants such as borderline SMOTE (BorSMOTE) \cite{han2005borderline} and Safe Level SMOTE (SL-SMOTE) \cite{bunkhumpornpat2009safe}.	
Due to space limit, we report the results of SVM, KNN, and RF since the have relatively better performance.
All settings remail the same except that SVM, KNN and RF are performed on data sets after up-sampling is pre-processed. 
We use a number of 5 neighbours for all the up-sampling methods to generate synthetic examples. 
For SVM and KNN, we merge the results of different kernels or distances and report whichever is higher. 
For DRM, we report the higher results as in \cref{tab_imbalanced}.
We report the classification results in \cref{tab_imbalanced_sampling}.
It is seen that with up-sampling, the performance of baseline methods are significantly improved whereas the DRM still shows comparable performance. 
It should be noted that the performance of DRM can be further improved if sampling techniques such as SMOTE is performed.
However, it is not within the scope of this paper to fully discuss sampling technique and thus we do not fully expand it. 
Also, it is worthy mentioning that the DRM itself has competitive performance compared with sampling technique + baseline classifiers.

	}

{
A deep thinking about the effectiveness of DRM on imbalanced data is as follows. 
It is seen the model of \cref{eq_DRMK} considers different weights on different groups with matrix $B$. 
The weight is determined by the group size as defined in (6). 
Thus, the weight can effectively alleviate the adverse effect of class imbalance in the data. 
In fact, the spirit of this approach is similar to algorithm-level methods for imbalanced data classification. 
Moreover, the decision rule in \cref{kmindist_decision} has clear geometrical interpretation such that class which the testing example truly belongs to is picked. 
Thus, from the model itself, it is convincing to expect promising performance of the DRM. }

\subsection{Abalation Study}
{ 
In this subsesction, we will conduct extensive experiments to illustrate the significance of the component terms in our model.
Without loss of generality, we conduct abalation study on imbalanced data sets for illustration.
Similar observations and conclusions can be seen and drawn on other types of data. 
Since the first term of \cref{eq_DRC} is the loss while the third term is known to avoid overfitting issue and strengthen robustness in ridge regression, 
we focus on the second term, i.e., $S_w(w)$. 
In particular, we will compare DRM with the classic kernel ridge regression (KRR) for illustration.
First, to show the significance of $S_w(w)$, we compare the best performance of KRR and DRM with all parameters selected by CV where the setting follows \cref{sec:app_small}. 
In particular, we show the results of using RBF and polynomial kernels separately  in \cref{fig_abalation_overall}. 
It is observed that the DRM has better performance than KRR, which reveals that the improvements benefit from the discriminative information introduced by $S_w(w)$. 
To further investigate how $S_w(w)$ performs, we compare KRR and DRM in the following way. We range the value of $\beta$ within $\{10^{-3},10^{-2},10^{-1},10^{0},10^{1},10^{2},10^{3}\}$. 
For each $\beta$ value, we report the best performances of KRR and DRM, where the other parameters are optimally selected among all possible combinations.
We separately  show the results of RBF and polynomial kernels in \cref{fig_abalation_alpha}.
It is observed that the performance curve of DRM is almost always above that of KRR, implying that with the term of $S_w(w)$ the DRM outperforms KRR with any $\beta$ value.
Thus, it is convincing to claim the significance of the terms in our model.

%


}
\section{Conclusion}
\label{sec_conclusion}
In this paper, we introduce a type of discriminative ridge regressions for supervised classification. 
This class of discriminative ridge regression models can be regarded as an extension to several existing regression models in the literature such as ridge, lasso, and grouped lasso. 
This class of new models explicitly incorporates discriminative information between different classes and hence is more suitable for classification than those existing regression models. 
As a special case of the new models, we focus on a quadratic model of the DRM as our classifier. 
The DRM gives rise to a closed-form solution, which is computationally suitable for high-dimensional data with a small sample size. 
For large-scale data, this closed-form solution is time demanding and thus we establish three iterative algorithms that have reduced computational cost. 
In particular, these three algorithms with the linear kernel all have linear cost in sample size and theoretically proven guarantee to globally converge.  
Extensive experiments on standard, real-world data sets demonstrate that the DRM outperforms several state-of-the-art classification methods, especially on high-dimensional data or imbalanced data; furthermore, the three iterative algorithms with the linear kernel all have desirable linear cost, global convergence, and scalability while being significantly more accurate than the linear SVM. {Ablation study confirms the significance of the key component in the proposed method.}

\bibliographystyle{IEEEtrans}
\bibliography{drm_tnnls_r1_v1_arxiv}  
%
%

\begin{IEEEbiographynophoto}
	{Chong Peng}
	received his PhD in Computer Science from Southern Illinois University, Carbondale, IL, USA in 2017. Currently, he is an assistant professor at College of Computer Science and Techonlogy, Qingdao University, China. He has published more than 30 research papers in top-tier conferences and journals, including
	AAAI, ICDE, CVPR, SIGKDD, ICDM, CIKM, ACM TIST, ACM TKDD, Neural Networks, IEEE TIP, TMM, and TGRS. His research interests include pattern recognition, data mining, machine learning, and computer vision. 
\end{IEEEbiographynophoto}

\begin{IEEEbiographynophoto}
	{Qiang Cheng}
	received the BS and MS degrees from the College of Mathematical
	Science at Peking University, China, and the PhD degree from the Department of
	Electrical and Computer Engineering at the University of Illinois, Urbana-Champaign.
	Currently, he is an associate professor at the Institute for Biomedical Informatics
	and the Department of Computer Science at the University of Kentucky. He
	previously worked as a faculty fellow at the Air Force Research Laboratory,
	Wright-Patterson, OH, and a senior researcher and senior research scientist at
	Siemens Medical Solutions, Siemens Corporate Research, Siemens Corp., Princeton,
	NJ. His research interests include data science, machine learning, pattern recognition, artificial
	intelligence, and biomedical informatics. He has published about 100 peer-reviewed papers in various
	premium venues including IEEE TPAMI, TNNLS, TSP, NIPS, CVPR, ICDM, AAAI, ICDE, ACM TIST, TKDD,
	and KDD. 
	He has a number of international patents issued or filed with Southern Illinois University
	Carbondale, IBM T.J. Watson Research Laboratory, and Siemens Medical.
\end{IEEEbiographynophoto}

\end{document}